\def\year{2022}\relax
\documentclass[letterpaper]{article} 
\usepackage{aaai22}  
\usepackage{times}  
\usepackage{helvet}  
\usepackage{courier}  
\usepackage[hyphens]{url}  
\usepackage{graphicx} 
\usepackage{natbib}  
\usepackage{caption} 
\DeclareCaptionStyle{ruled}{labelfont=normalfont,labelsep=colon,strut=off} 
\usepackage{algorithm}
\usepackage{algorithmic}

\newtheorem{proposition}{Proposition}
\newtheorem{definition}{Definition}
\newtheorem{theorem}{Theorem}
\newtheorem{lemma}{Lemma}
\newenvironment{thmbis}[1]
  {\renewcommand{\theorem}{\textbf{Theorem \ref{#1}}}%
   \begin{theorem}}
  {\end{theorem}}
  \newenvironment{propbis}[1]
  {\renewcommand{\proposition}{\textbf{Proposition \ref{#1}}}%
   \begin{proposition}}
  {\end{proposition}}
\newenvironment{proof}{{\indent \indent \it \textbf{Proof:}\ }}{\hfill $\blacksquare$\par}

\usepackage{amsmath}
\usepackage{amssymb}
\newcommand{\R}{\mathbb{R}}
\newcommand{\W}{\mathbf{W}}
\newcommand{\w}{\mathbf{w}}
\newcommand{\K}{\mathbf{K}}
\newcommand{\J}{\mathbf{J}}
\newcommand{\D}{\mathcal{D}}
\newcommand{\x}{\mathbf{x}}
\newcommand{\X}{\mathcal{X}}
\newcommand{\Y}{\mathcal{Y}}
\newcommand{\Loss}{\mathcal{L}}
\newcommand{\I}{\mathbb{I}}
\newcommand{\df}{\operatorname{d}}

\usepackage{newfloat}
\usepackage{listings}
\floatstyle{ruled}
\newfloat{listing}{tb}{lst}{}
\floatname{listing}{Listing}
\title{Rethinking Influence Functions of Neural Networks\\ in the Over-parameterized Regime}
\author{
    Rui Zhang,\textsuperscript{\rm 1,2} Shihua Zhang\textsuperscript{\rm 1,2}\thanks{Shihua Zhang is the corresponding author.}
}
\affiliations{
    \textsuperscript{\rm 1}NCMIS, CEMS, RCSDS, Academy of Mathematics and Systems Science, Chinese Academy of Sciences, Beijing, China\\
    \textsuperscript{\rm 2}School of Mathematical Sciences, University of Chinese Academy of Sciences, Beijing, China\\

    \{rayzhang, zsh\}@amss.ac.cn
}

\usepackage{bibentry}

\begin{document}

\maketitle

\begin{abstract}
Understanding the black-box prediction for neural networks is challenging. To achieve this, early studies have designed influence function (IF) to measure the effect of removing a single training point on neural networks. However, the classic implicit Hessian-vector product (IHVP) method for calculating IF is fragile, and theoretical analysis of IF in the context of neural networks is still lacking. To this end, we utilize the neural tangent kernel (NTK) theory to calculate IF for the neural network trained with regularized mean-square loss, and prove that the approximation error can be arbitrarily small when the width is sufficiently large for two-layer ReLU networks. We analyze the error bound for the classic IHVP method in the over-parameterized regime to understand when and why it fails or not. In detail, our theoretical analysis reveals that (1) the accuracy of IHVP depends on the regularization term, and is pretty low under weak regularization; (2) the accuracy of IHVP has a significant correlation with the probability density of corresponding training points. We further borrow the theory from NTK to understand the IFs better, including quantifying the complexity for influential samples and depicting the variation of IFs during the training dynamics. Numerical experiments on real-world data confirm our theoretical results and demonstrate our findings.
\end{abstract}  %

\section{Introduction}
Influence function \citep{doi:10.1080/01621459.1974.10482962} is a classic technique from robust statistics, which measures the effect of changing a single sample point on an estimator. \citet{pmlr-v70-koh17a} transferred the concept of IFs to understanding why neural networks make corresponding predictions. IF is one of the most common approaches in explainable AI and is widely used for boosting model performance \citep{wang2018data}, measuring group effects \citep{NEURIPS2019_a78482ce}, investigating model bias \citep{pmlr-v97-wang19l}, understanding generative models \citep{kong2021understanding} and so on. Specifically, IF reflects the effect of removing one training point on a neural network’s prediction, thus can be used to discover the most influential training points for a given prediction. In their work, the implicit Hessian-vector product (IHVP) was utilized to estimate the IF for neural networks. However, the numerical experiments in \citep{basu2021influence} pointed out that IFs calculated via IHVPs are often erroneous for neural networks. Furthermore, theoretical understanding for why these phenomena happened still lacks in the neural network regime.

To theoretically understand the IF for neural networks, we need to overcome the non-linearity and over-parameterized properties in neural networks. Fortunately, recent advances of NTK \cite{NEURIPS2018_5a4be1fa,NEURIPS2019_0d1a9651} shed light on the theory of over-parameterized neural networks. The key idea of NTK is that an infinitely wide neural network trained by gradient descent is equivalent to kernel regression with NTK. Remarkably, the theory of NTK builds a bridge between the over-parameterized neural networks and the kernel regression method, which dramatically reduces the difficulty of analyzing the neural networks theoretically. With the help of NTK, the theory for over-parameterized neural networks has achieved rapid progresses \cite{pmlr-v97-arora19a,NEURIPS2019_663fd3c5,Hu2020Simple,zhang2021how}, which encourage us to deal with the puzzle of calculating and understanding the IFs for neural networks in the NTK regime.

In this work, we utilize the NTK theory to calculate IFs and analyze the behavior theoretically for the over-parameterized neural networks. In summary, we make the following contributions:
\begin{itemize}
    \item We utilize the NTK theory to calculate IFs for over-parameterized neural networks trained with regularized mean-square loss, and prove that the approximation error can be arbitrarily small when the width is sufficiently large for two-layer ReLU networks. Remarkably, we prove the first rigorous result to build the equivalence between the fully-trained neural network and the kernel predictor in the regularized situation. Numerical experiments confirm that IFs calculated in the NTK regime can approximate the actual IFs with high accuracy.
    \item We analyze the error bound for the classic IHVP method in the over-parameterized regime to understand when and why it fails or not. On the one hand, our bounds reveal that the accuracy of IHVP depends on the regularization term which was only characterized before by numerical experiments in \citep{basu2021influence}. On the other hand, we theoretically prove that the accuracy of IHVP has a significant correlation with the probability density of corresponding training points, which has not been revealed in previous literature. Numerical experiments verify our bounds and statements.
    \item Furthermore, we borrow the theory from NTK to understand the behavior of IFs better. On the one hand, we utilize the theory of model complexity in NTK to quantify the complexity for influential samples and reveal that the most influential samples make the model more complicated. On the other hand, we track the dynamic system of the neural networks and depict the variation of IFs during the training dynamics.
\end{itemize}

\section{Related Works}
\subsection{Influence Functions in Machine Learning}
 To explain the black-box prediction in neural networks, \citet{pmlr-v70-koh17a} utilized the concept of IF to trace a model’s predictions through its learning algorithm and back to the training data. To be specific, they considered the following question:
\emph{How would the model’s predictions change if we did not have this training point?}

To calculate the IF of a training point for neural networks, \citet{pmlr-v70-koh17a} proposed an approximate method based on IHVP and \citet{DBLP:conf/aaai/ChenLYWM21} considered the training trajectory to avoid the calculation of Hessian matrix. However, \citet{basu2021influence} figured out that the predicting precision via IHVP may become particularly poor under certain conditions for neural networks. 

\subsection{Theory and Applications of NTK}
In the last few years, several papers have shown that the infinite-width neural network with the square loss during training can be characterized by a linear differential equation \cite{NEURIPS2018_5a4be1fa,NEURIPS2019_0d1a9651,pmlr-v97-arora19a}. In particular, when the loss function is the mean square loss, the inference performed by an infinite-width neural network is equivalent to the kernel regression with NTK. The progresses about NTK shed light on the theory of over-parameterized neural networks and were utilized to understand the optimization and generalization for shallow and deep neural networks \citep{pmlr-v97-arora19a,NEURIPS2019_cf9dc5e4, pmlr-v97-allen-zhu19a, NEURIPS2020_8abfe8ac}, regularization methods
\citep{Hu2020Simple}, and data augmentation methods \citep{li2019enhanced,zhang2021how} in the over-parameterized regime. NTK can be analytically calculated using exact Bayesian inference, which has been implemented in NEURAL TANGENTS for working with infinite-width networks efficiently \citep{Novak2020Neural}. 

In this work, we will firstly give rigorous prove to reveal the equivalence between the regularized neural networks and the kernel ridge regression predictor via NTK in the over-parameterized regime, then utilize the tools about NTK to calculate and better understand the properties of IFs.

\section{Preliminaries}
\subsection{Notations}
We utilize bold-faced letters for vectors and matrices. For a matrix $\mathbf{A}$, let $[\mathbf{A}]_{i j}$ be its $(i, j)$-th entry and $\operatorname{vec} (\mathbf{A})$ be its vectorization. For a vector $\mathbf{a}$, let $[\mathbf{a}]_i$ be its $i$-th entry. We use $\|\cdot\|_{2}$ to denote the Euclidean norm of a vector or the spectral norm of a matrix, and use $\|\cdot\|_{F}$ to denote the Frobenius norm of a matrix. We use $\langle\cdot,\cdot\rangle$ to denote the standard Euclidean inner product between two vectors or matrices. Let $\mathbf{I}_n$ be an $n\times n$ identity matrix, and $\mathbf{e}_i$ denote the $i$-th unit vector and $[n]=\{1,2, \cdots, n\}$. For a set $A$, we utilize $\mathrm{unif}(A)$ to denote the uniform distribution over $A$. We utilize $\mathbb{I}(\cdot)$ to denote the indicator function. We utilize $f^{\backslash i}$ to denote the network or kernel predictor trained without the $i$-th training point. To be clear, we respectively define $f_{nn}$ and $f_{ntk}$ as neural network (nn) and its corresponding NTK predictor. We further denote $\W(t)$ as the parameters of neural networks at time $t$ during the training process. 

\subsection{Network Models and Training Dynamics}
In this paper, we consider the two-layer neural networks with rectified
linear unit (ReLU) activation:
\begin{equation}
  f_{nn}(\mathbf{x})=\frac{1}{\sqrt{m}} \sum_{r=1}^{m} a_{r} \sigma\left(\mathbf{w}_{r}^{\top} \mathbf{x}\right) \mathbf{,}  
  \label{nn}
\end{equation}
where $\mathbf{x} \in \R^d$ is the input, $\mathbf{W} = [\mathbf{w}_1, \cdots, \mathbf{w}_m] \in \R^{d\times m}$ is the weight matrix in the first layer, and $\mathbf{a} = [a_1, \cdots, a_m]^{\top} \in \R^{m}$ is the weight vector in the second layer. We initialize the parameters randomly as follows:
$$
\mathbf{w}_{r}(0) \sim \mathcal{N}\left(\mathbf{0}, \kappa^{2} \mathbf{I}_d\right), \ a_{r}(0) \sim \mathrm{unif}(\{-1,1\}), \quad \forall r \in[m],
$$
where $0 < \kappa \ll 1$ controls the magnitude of initialization, and all randomnesses are independent. For simplicity, we fix the second layer $\mathbf{a}$ and only update the first layer $\W$ during training. The same setting has been used in \cite{pmlr-v97-arora19a,du2018gradient}.
 
We are given $n$ training points $(\X,\Y) = \{\x_i,y_i\}_{i=1}^n$ drawn i.i.d. from an underlying data distribution $\mathcal{D}$ over $\R^d\times\R$. For simplicity, we assume that for each $(\x, y)$ sampled from $\mathcal{D}$ satisfying $\|\x\|_2 = 1$ and $|y| \leq 1$.

To study the effect of regularizer on calculating the IF, we train the neural networks through gradient descent on the regularized mean square error loss function as follows:
\begin{equation}
    \mathcal{L}(\W)=\frac{1}{2} \sum_{i=1}^{n}\left(f_{nn}\left(\x_i; \W\right)-{y}_{i}\right)^{2}+\frac{\lambda}{2}\|\W-\W(0)\|_F^{2}, \label{loss}
\end{equation}
where the regularizer term restricts the distance between the network parameters to initialization, and has been previously studied in \cite{Hu2020Simple}. And we consider minimizing the loss function $\Loss(\W)$ in the gradient flow regime, i.e., gradient descent with infinitesimal step size, then the evolution of $\W(t)$ can be described by the following ordinary differential equation:

\begin{equation}
\begin{aligned}
    &\frac{\df\W(t)}{\df t} =- \frac{\partial \Loss(\W(t))}{\partial \W(t)} \\ 
    = & 
    \sum_{i=1}^{n} (y_i - f_{nn}(\x_i;\W(t)))\frac{\partial f_{nn}(\x_i;\W(t))}{\partial \W(t)} - \lambda (\W(t)- \W(0)).
\end{aligned}\label{W_ode}
\end{equation}

\subsection{NTK for Two-layer ReLU Neural Networks}
Given two data points $\x$ and $\x^{\prime}$, the NTK associated with two-layer ReLU neural networks has a closed form expression as follows \cite{pmlr-v54-xie17a,du2018gradient}:
\begin{equation}
\begin{aligned}
\K^{\infty} (\x,\x^{\prime})&\triangleq  
\lim_{m \to \infty}
\left\langle\frac{\partial f_{nn}(\x;\W(0))}{\partial \W(0)},\frac{\partial f_{nn}(\x^{\prime};\W(0))}{\partial \W(0)}\right\rangle
\\&= \mathbb{E}_{\mathbf{w} \sim \mathcal{N}(\mathbf{0}, \mathbf{I}_d)}\left[\mathbf{x}^{\top} \mathbf{x}^{\prime} \mathbb{I}\left\{\mathbf{w}^{\top} \mathbf{x} \geq 0, \mathbf{w}^{\top} \mathbf{x}^{\prime} \geq 0\right\}\right] \\
&=\frac{\mathbf{x}^{\top} \mathbf{x}^{\prime}\left(\pi-\arccos \left(\mathbf{x}^{\top} \mathbf{x}^{\prime}\right)\right)}{2 \pi}.
\end{aligned}\label{ntk_closed}
\end{equation}

Equipped with the NTK function, we consider the following kernel ridge regression problem:
\begin{equation}
    \min_{\boldsymbol{\beta}\in\R^{n}}\ \frac{1}{2}\|\Y-\K^{\infty}_{tr}\boldsymbol{\beta}\|_2^2+\frac{\lambda}{2}\boldsymbol{\beta}^{\top}{\K}_{tr}^{\infty}\boldsymbol{\beta} \label{krr},
\end{equation}
where $\K_{tr}^{\infty}\in \R^{n \times n}$ is the NTK matrix evaluated on the training data, i.e., $[\K_{tr}^{\infty}]_{i,j} = \K^{\infty}(\x_i,\x_j)$, and $\boldsymbol{\beta}^{*} \triangleq (\K_{tr}^{\infty}+\lambda \mathbf{I}_n)^{-1}\Y$ is the optimizer of the problem \eqref{krr}. Hence the prediction of kernel ridge regression using
NTK on a test point $\x_{te}$ is:
\begin{equation}
    f_{ntk}(\x_{te}) = (\K_{te}^{\infty})^{\top} (\K_{tr}^{\infty}+\lambda \mathbf{I}_n)^{-1}\Y, \label{krr_pred}
\end{equation}
 where $\K_{te}^{\infty} \in \R^{n}$ is the NTK evaluated between the test point $\x_{te}$ and the training points $\X$, i.e., $[\K_{te}]_i = \K^{\infty}(\x_{te},\x_i)$.

\section{IFs for Over-parameterized Neural Networks}
In this section, our goal is to evaluate the effect of a single training point for the over-parameterized neural network’s predictions. In detail, given a training point $\x_i \in \X$, we need to calculate the variation of test loss after removing $\x_i$ from $\X$ for the neural network $f_{nn}$, which is denoted by $\mathcal{I}_{nn}(\x_i,\x_{te})$ as follows:
\begin{equation}
\mathcal{I}_{nn}(\x_i,\x_{te}) \triangleq \frac{1}{2}(f_{nn}^{\backslash i}(\x_{te})-y_{te})^2-\frac{1}{2}(f_{nn}(\x_{te})-y_{te})^2.
\end{equation}

To calculate $\mathcal{I}_{nn}(\x_i,\x_{te})$ for all $\x_{i}$, it is impossible to retrain the neural network one by one. Furthermore, due to the fact that the neural network is over-parameterized, it is prohibitively slow to calculate the IF via the IHVP method \citep{pmlr-v70-koh17a}. In this work, we utilize the corresponding NTK predictor $f_{ntk}$ to approximate the behavior of $f_{nn}$, hence we define the IF in the NTK regime as follows:
\begin{equation}
       \mathcal{I}_{ntk}(\x_i,\x_{te}) \triangleq \frac{1}{2}(f_{ntk}^{\backslash i}(\x_{te})-y_{te})^2-\frac{1}{2}(f_{ntk}(\x_{te})-y_{te})^2.
\end{equation}

To calculate $f_{ntk}^{\backslash i}(\x_{te})$ efficiently, we borrow the technique from the online Gaussian regression to update the inverse of kernel matrix \cite{6790153,1315946}. Then we have:
\begin{equation}
    f_{ntk}^{\backslash i}(\x_{te}) = (\K_{te}^{\infty})^{\top}\left((\K_{tr}^{\infty}+\lambda \mathbf{I}_n)^{-1} - \frac{\mathbf{k}_{-i}\mathbf{k}_{-i}^{\top}}{k_{-ii}}\right) \Y,
\end{equation}
where $\mathbf{k}_{-i}$ and $k_{-ii}$ denote the $i$-th column and the $(i,i)$-th entry of $(\K_{tr}^{\infty}+\lambda \mathbf{I}_n)^{-1}$ respectively.

After giving a calculation method for $\mathcal{I}_{ntk}(\x_{i},\x_{te})$, we show that $\mathcal{I}_{ntk}(\x_{i},\x_{te})$ can be a good approximation of $\mathcal{I}_{nn}(\x_{i},\x_{te})$ in the over-parameterized regime. Note that in the ridgeless regime, i.e., $\lambda = 0$,  \citet{NEURIPS2019_dbc4d84b} rigorously showed the equivalence between a fully-trained neural network and its corresponding kernel predictor. However, in the kernel ridge regime, i.e., $\lambda > 0$, 
there is no rigorous result so far. In this paper, we propose the following theorem which reveals the equivalence between the two-layer fully-trained wide ReLU neural network $f_{nn}$ and its corresponding NTK predictor $f_{ntk}$ in the ridge regression regime.
\begin{theorem}\label{theorem1}
Suppose $0<\lambda< n^{\frac{1}{2}},\ \kappa = \mathcal{O}\left(\frac{\sqrt{\delta} \lambda \epsilon}{n}\right)$ and $\ m = \Omega\left(\frac{n^8}{\kappa^2\epsilon^2\delta^4\lambda^6}\right)$, then for any $\x_{te} \in \mathbb{R}^{d}$ with $\|\x_{t e}\|_2=1$, with probability at least $1 - \delta$ over random initialization, we have:
\begin{equation}
    |f_{nn}(\x_{te}) - f_{ntk}(\x_{te})|= \mathcal{O}(\epsilon).
\end{equation}
\end{theorem}

For the proof of Theorem \ref{theorem1}, we first analyze the evolution of $\W(t)$ and prove the perturbation of parameters are small during training. Then we follow the ideas from \cite{du2018gradient,pmlr-v97-arora19a} to prove the perturbation of the kernel matrix can be bounded by the perturbation of parameters during training, and build the equivalence between the wide neural network and its linearization. After that, we borrow the lemma from \citep{pmlr-v97-arora19a} to establish the equivalence between the linearization of neural network and NTK predictor, then the theorem can be proved via triangle inequality. See Appendix B for the proof.

After proving Theorem \ref{theorem1}, the approximation error can be evaluated via simple analysis, which is shown in the following theorem. See Appendix B for the proof.

\begin{theorem}\label{theorem2}
Suppose $0<\lambda< n^{\frac{1}{2}},\ \kappa = \mathcal{O}\left(\frac{\sqrt{\delta} \lambda \epsilon}{n}\right)$, and $\ m = \Omega\left(\frac{n^8}{\kappa^2\epsilon^2\delta^4\lambda^6}\right)$,  $f_{nn}$ and $f_{ntk}$ are uniformly bounded over the unit sphere by a constant $C$, i.e. $|f_{nn}(\x)|<C$ and $|f_{ntk}(\x)|<C$ for all $\|\x\|_2 =1$. Then for any training point $\x_i \in \X$ and test point $\x_{te} \in \mathbb{R}^{d}$ with $\|\x_{t e}\|_2=1$, with probability at least $1 - \delta$ over random initialization, we have:
\begin{equation}
    |\mathcal{I}_{nn}(\x_i,\x_{te}) - \mathcal{I}_{ntk}(\x_i,\x_{te})|= \mathcal{O}(\epsilon).
\end{equation}
\end{theorem}

Theorem \ref{theorem2} reveals that the IF calculated in the NTK regime can be arbitrarily close to the actual ones with high probability as long as the hidden layer is sufficiently wide. We compare the IFs calculated in the NTK regime and the ones obtained via leave-one-out retraining to verify our theory. In particular, we evaluate our method on MNIST \cite{726791} and CIFAR-10 \cite{krizhevsky2009learning} for two-layer ReLU neural networks with the width from $10^4$ to $8\times 10^4$ respectively. Details of experimental settings and more experiments can be seen in the Appendix A. The numerical results are shown in Figure \ref{calif_1} and Table \ref{calif_2} respectively. We find that the predicted IFs are highly close to the actual ones, with the Pearson correlation coefficient ($R$) and Spearman’s rank correlation coefficient ($\rho$) greater than 0.90 in general, and the approximate accuracy is increasing with the width of neural networks, which is consistent with Theorem \ref{theorem2}. 
\begin{figure}
\includegraphics[width=0.48\textwidth]{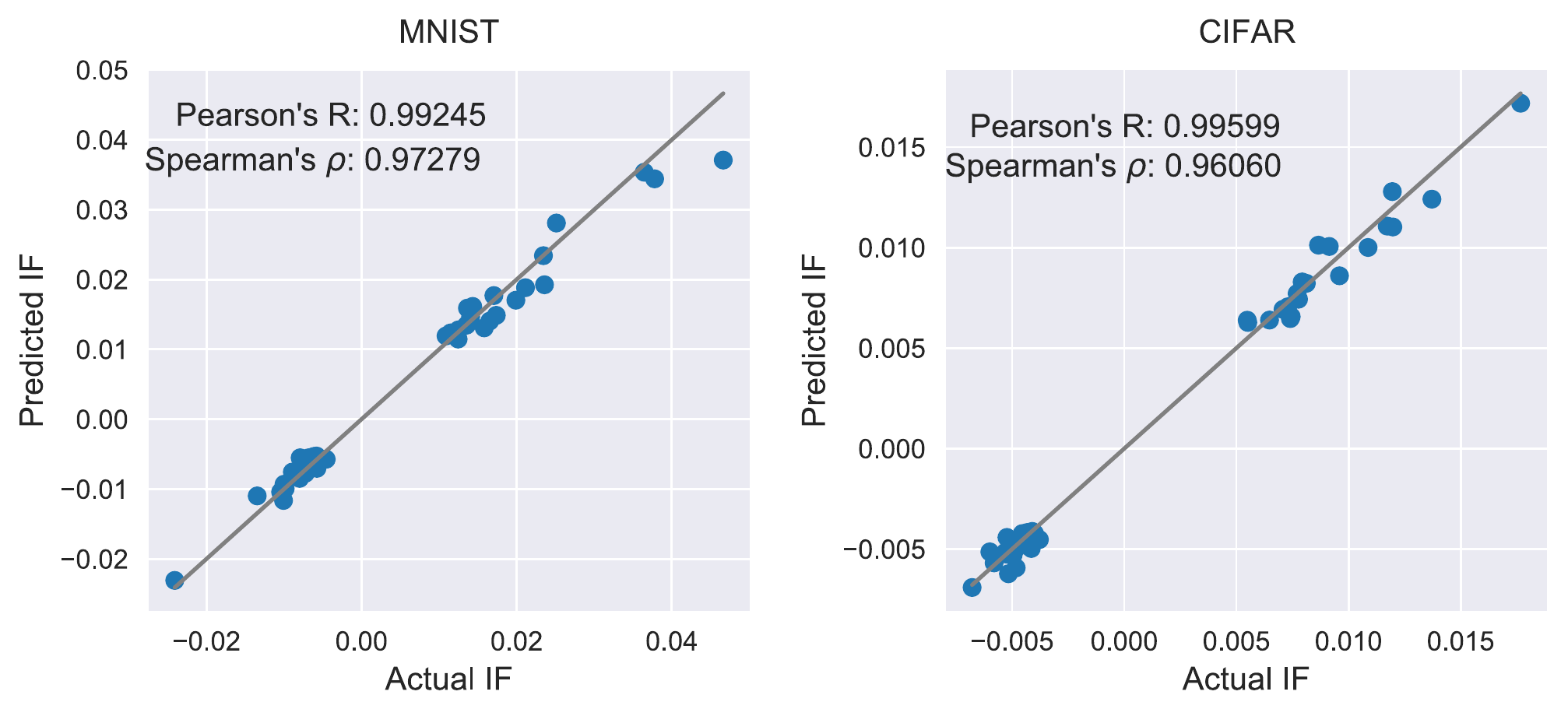} 
\caption{Comparison of predicted IFs and actual ones on the 40
most influential training points for a two-layer ReLU networks with $8\times 10^4$ neurons in the hidden layer.} 
\label{calif_1} 
\end{figure}

\begin{table}[]
\begin{tabular}{|l|c|c|c|c|}
\hline
\multicolumn{5}{|c|}{MNIST}                           \\ \hline
Width ($10^4$)    & 8      & 4      & 2      & 1      \\ \hline
Pearson's R       & 0.9955 & 0.9866 & 0.9802 & 0.9766 \\ \hline
Spearman's $\rho$ & 0.9728 & 0.9565 & 0.9531 & 0.9341 \\ \hline
\multicolumn{5}{|c|}{CIFAR}                           \\ \hline
Width ($10^4$)    & 8      & 4      & 2      & 1      \\ \hline
Pearson's R       & 0.9960 &   0.9904     & 0.9808 & 0.9683 \\ \hline
Spearman's $\rho$ & 0.9606 &    0.9146    & 0.8672 & 0.8583 \\ \hline
\end{tabular}
\caption{The correlation coefficients between actual and predicted IFs on MNIST and CIFAR respectively.} \label{calif_2}
\end{table}

\section{Error Analysis for the IHVP Method}
In this section, we aim to analyze the approximation error of the IHVP method \citep{pmlr-v70-koh17a} when calculating the IF in the over-parameterized regime. Our analysis also reveals when and why IHVP can be accurate or not, which theoretically explains the phenomenon that the regularization term controls the estimation accuracy of IFs proposed in \citep{basu2021influence}, and also brings new insights into the understanding of IHVP.
\subsection{The IHVP Method in Over-parameterized Regime}
Formally, we utilize $\hat{\mathcal{I}}_{nn}(\x_i,\x_{te})$ to denote the IF calculated by IHVP, which is defined as follows:
\begin{equation}\label{hessian-vector}
    \hat{\mathcal{I}}_{nn}(\x_i,\x_{te}) = \operatorname{vec}\left(\frac{\partial \ell(\x_{te})}{\partial \W}\right)^{\top}H(\infty)^{-1}\operatorname{vec}\left(\frac{\partial \ell(\x_{i})}{\partial \W}\right),
\end{equation}
where $\ell(\x) \triangleq \frac{1}{2}(f_{nn}(\x)-y) ^2$ and $H(\infty)\triangleq \sum_{i=1}^{n} \nabla_{\W}^{2} \ell\left(\x_{i}; \W(\infty)\right)+\lambda\mathbf{I}_{md}$ denote the mean-square loss and the Hessian matrix at the end of training respectively. Remarkably, the calculation of Equation (\ref{hessian-vector}) is infeasible due to the size of modern neural networks. To solve this problem, \citet{pmlr-v70-koh17a} and \citet{DBLP:conf/aaai/ChenLYWM21} utilized stochastic estimation or hypergradient-based optimization to calculate Equation (\ref{hessian-vector}) respectively. However, the gap between $\hat{\mathcal{I}}_{nn}(\x_i,\x_{te})$ and the actual IF is inevitable in general, and cannot be avoided by these methods. Thus in this work, we only analyze the error between $\hat{\mathcal{I}}_{nn}(\x_i,\x_{te})$ and the actual IF in the over-parameterized regime. The following proposition shows that $ \hat{\mathcal{I}}_{nn}(\x_i,\x_{te})$ tends to $ \hat{\mathcal{I}}_{ntk}(\x_i,\x_{te})$ when the width $m$ goes to infinity, where $ \hat{\mathcal{I}}_{ntk}(\x_i,\x_{te})$ denotes the IF calculated via IHVP method in the NTK regime.
\begin{proposition}\label{pop1}
For any $\x_{te} \in \R^d$ with $\|\x_{te}\|_2=1$, let the width $m \to \infty$, then with
probability arbitrarily close to 1 over random initialization, we have:
\begin{equation}\label{equa_if_hat}
\begin{aligned}
    \hat{\mathcal{I}}_{nn}(\x_i,\x_{te}) \to& 
    \alpha(\x_i,\x_{te})(f_{ntk}(\x_{te})-y_{te})(f_{ntk}(\x_{i})-y_{i}) \\ \triangleq &\hat{\mathcal{I}}_{ntk}(\x_i,\x_{te}),
\end{aligned}
\end{equation}
where $\alpha(\x_i,\x_{te}) \triangleq (\mathbf{K}_{te}^{\infty})^{\top}(\mathbf{K}_{tr}^{\infty}+\lambda \mathbf{I}_n)^{-1}\mathbf{e}_i$.
\end{proposition}

For simplicity of analysis, we rewrite $\mathcal{I}_{ntk}(\x_i,\x_{te})$ as follows:
\begin{proposition}
For any $\x_{te} \in \R^d$ with $\|\x_{te}\|_2=1$ and $\x_i \in \X$, we have:
\begin{equation}\label{equa_if}
\begin{aligned}
    \mathcal{I}_{ntk}(\x_i,\x_{te}) &=  \underbrace{\alpha(\x_i,\x_{te})(f_{ntk}(\x_{te})-y_{te})(f_{ntk}^{\backslash i}(\x_{i})-y_{i})}_{\text{\uppercase\expandafter{\romannumeral1}}} \\ 
    &+\underbrace{\frac{1}{2}\alpha(\x_i,\x_{te})^2(f_{ntk}^{\backslash i}(\x_i) - y_i)^2}_{\text{\uppercase\expandafter{\romannumeral2}}}.
\end{aligned}
\end{equation}\label{pop2}
\end{proposition}

The proof of these two propositions can be seen in the Appendix C. It is worth mentioning that the expectation of $\left|f_{ntk}(\x_{te}) - y_{te}\right|$ and $\left|f_{ntk}^{\backslash i}(\x_{i}) - y_{i}\right|$ both represent the generalization error of the model \citep{elisseeff2003leave}. And $\alpha(\x_i,\x_{te})$ represents the coefficient corresponding to $\x_i$ when projecting $\mathbf{J}({\x_{te}})$ onto $\mathbf{J}({\X})$, where $\mathbf{J}(\cdot)$ denotes the feature map in NTK. In general, $\alpha(\x_i,\x_{te})$ is distinctly smaller than $1$. Thus we have $|\text{term   }\text{\uppercase\expandafter{\romannumeral1}}| \gg |\text{term} \ \text{{\uppercase\expandafter{\romannumeral2}}}|$, which means it is reasonable to approximate $ \mathcal{I}_{ntk}(\x_i,\x_{te})$ as follows:
\begin{equation}\label{equa_approx}
     \mathcal{I}_{ntk}(\x_i,\x_{te}) \approx {\alpha(\x_i,\x_{te})(f_{ntk}(\x_{te})-y_{te})(f_{ntk}^{\backslash i}(\x_{i})-y_{i})}.
\end{equation}

By Comparing Equation (\ref{equa_approx}) with Equation (\ref{equa_if_hat}), we can obtain that the main approximation error is caused by the variance between $f^{\backslash i}_{ntk}(\x_i)$ and $f_{ntk}(\x_i)$, which plays a key role when bounding the approximation error of IHVP. In the following two subsections, we reveal why and how the regularization term and the probability density of the corresponding training points control the approximation error.

\subsection{The Effects of Regularization}
Although previous literature \citep{basu2021influence} has pointed out that the regularization term is essential to get high-quality IF estimates via numerical experiments, this phenomenon is not well-understood for neural networks in theory. In this subsection, we prove that the lower bound of the approximation error is controlled by the regularization parameter $\lambda$, and the least eigenvalue of NTK also plays a key role. The following theorem reveals the relationship between the approximation error and the regularization term for over-parameterized neural networks. See Appendix C for the proof.
\begin{theorem}
Given $\x_i \in \X$ and $\x_{te} \in \R^d$ with $\|\x_{te}\|_2=1$, we have:
\begin{equation}
\begin{aligned}
     &\left|\mathcal{I}_{ntk}(\x_i,\x_{te})- \hat{\mathcal{I}}_{ntk}(\x_i,\x_{te})\right| \\\geq& \frac{\lambda_{min}}{\lambda_{min}+\lambda}|\mathcal{I}_{ntk}(\x_i,\x_{te})|-\frac{1}{2}\alpha(\x_i,\x_{te})^2(f_{ntk}^{\backslash i}(\x_i) - y_i)^2,
\end{aligned}
\end{equation}
where $\lambda_{min}$ is the least eigenvalue of $\K_{tr}^{\infty}$. Furthermore, if $|{\alpha}(\x_i,\x_{te})| \leq \sqrt{\frac{\gamma}{n}}\|\boldsymbol{\alpha}(\x_{te}) \|_2$ for some $\gamma>0$, where $\boldsymbol{\alpha}(\x_{te}) \triangleq (\mathbf{K}_{te}^{\infty})^{\top}(\mathbf{K}_{tr}^{\infty}+\lambda \mathbf{I}_n)^{-1}$, and $\mathbb{E}_{\x\sim\mathcal{D}}[(f_{ntk}(\x)-y)^2]=\mathcal{O}({\sqrt{1/n}})$, then with probability at least $1-\delta$ we have:
\begin{equation}
\begin{aligned}
  &\left|\mathcal{I}_{ntk}(\x_i,\x_{te})- \hat{\mathcal{I}}_{ntk}(\x_i,\x_{te})\right|\\\geq& \frac{\lambda_{min}}{\lambda_{min}+\lambda}|\mathcal{I}_{ntk}(\x_i,\x_{te})|-\mathcal{O}(\frac{\gamma}{\delta\lambda n^{3/2}}).
\end{aligned}
\end{equation}\label{lowerbound}
\end{theorem}

\noindent\textbf{Remark:} Similar to Theorem \ref{lowerbound}, we can give a upper bound of the error rate controlled by $\frac{\lambda_{max}}{\lambda_{max}+\lambda}$, where $\lambda_{max}$ is the largest eigenvalue of $\K_{tr}^{\infty}$. However, $\lambda_{max}$ is of order $\mathcal{O}{(n)}$ \citep{pmlr-v108-li20j}, hence the term $\frac{\lambda_{max}}{\lambda_{max}+\lambda}$ will close to $1$ and be meaningless in general.  
\citet{pmlr-v97-arora19a} proved that the generalization error $\mathbb{E}_{\x\sim\mathcal{D}}[(f_{ntk}(\x)-y)^2]$ is in the order of $\mathcal{O}({\sqrt{1/n}})$ for two-layer ReLU networks when the data is generated by certain functions in Theorem 5.1. Hence this assumption is reasonable.

Theorem \ref{lowerbound} reveals that the error rate can be lower bounded by $\frac{\lambda_{min}}{\lambda_{min}+\lambda}$. To verify our lower bound, we compare the theoretical lower bound of error rates with the mean actual error rates for different $\lambda$, ranging from $2^{-4}$ to $2^4$ respectively. Numerical experiment results reveal our bound can reflect the performance of IHVP well in the NTK regime (Figure \ref{lower_bound}).
\begin{figure}
\includegraphics[width=0.48\textwidth]{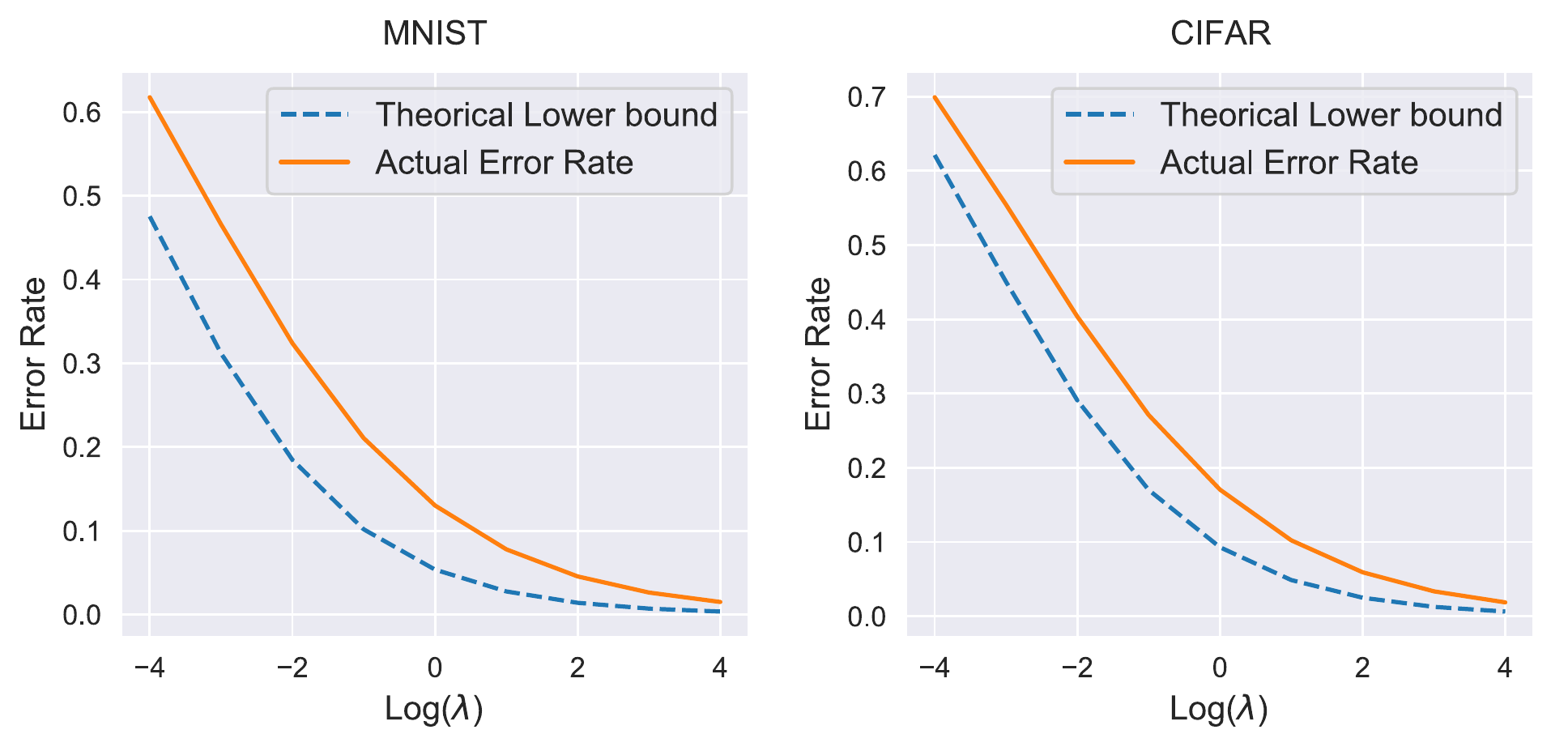}\centering
\caption{Comparison of the theoretical lower bound of error rates and the actual error rates for different $\lambda$ on MNIST and CIFAR respectively.} 
\label{lower_bound} 
\end{figure}

\subsection{ The Effects of Training Points}
In this subsection, we show that the approximation error is highly related to the probability density of corresponding training point, which has not been clearly revealed in previous literature. To model this phenomenon, we firstly assume the data sampled from the following finite mixture distribution.

\begin{definition}\label{assump}
Consider the data is generated from a finite mixture distribution $\mathcal{D}$ as follows. There are $K$ unknown distributions $\{\D_{k}\}_{k=1}^{K}$ over $\R^d$ with probabilities $\{p_k\}_{k=1}^{K}$ respectively, such that $\sum_{k=1}^{K}p_k=1$. Assume we sample $n_k$ data points from $\mathcal{D}_k$ respectively such that $\sum_{k=1}^{K}n_k = n$, and let $\X_k$ be the training data sampled from $D_k$. Let us define the support of a distribution $\mathcal{D}$ with density $p$ over $\R^{d}$ as $\operatorname{supp}(\mathcal{D}) \triangleq \{x: p(x)>0\}$, and the radius of a distribution $\mathcal{D}$ over  $\R^{d}$ as $r\left(\mathcal{D}\right)\triangleq \max\left\{\|\x-\mathbb{E}_{\mathbf{z}\sim \mathcal{D}}[\mathbf{z}]\|_2:\x\in \operatorname{supp}(\mathcal{D})\right\} .$ Then we make the following assumptions about the data.

\textbf{(A1)} (Sample Proportion) We have $n_k = p_k\cdot n$ for all $k\in [K]$ . 

\textbf{(A2)} (Uniformly Bounded) There exists a constant $r>0$, such that $r(\mathcal{D}_k)<r$ for all $k \in [K]$.  
\end{definition}

The above assumptions of data follow that of the previous works in \cite{NEURIPS2018_54fe976b,dong2019distillation,pmlr-v108-li20j}. We prove the following theorem which reveals the relationship between the approximation error for IFs of $\x_i\in \X_k$ and its corresponding probability density $p_k$. See Appendix C for the proof.

\begin{theorem}\label{theorem4}
Consider a dataset $(\X,\Y) = \{\x_i,y_i\}_{i=1}^n$ generated from the finite mixture model described in Definition \ref{assump}. Let $\boldsymbol{\alpha}(\x_i) \triangleq (\mathbf{K}_{tr}^{\infty})^{\top}(\mathbf{K}_{tr}^{\infty}+\lambda \mathbf{I}_n)^{-1}\mathbf{e}_i$. Suppose $|[\boldsymbol{\alpha}(\x_i)]_i|\leq \sqrt{\frac{\gamma}{n}}\|\boldsymbol{\alpha}(\x_{i}) \|_2$ for some $\gamma>0$, and $\lambda > \frac{\sqrt{2}\lambda_{max}\epsilon_r}{1-\sqrt{2}\epsilon_r}$, where $\epsilon_r^2 \triangleq 2r^2+\arccos(1-2r^2)$ is a small constant. Then for $\x_i \in \X_k$ and $\x_{te}\sim \mathcal{D}$, we have:
\begin{equation}
\begin{aligned}
         & \left|\mathcal{I}_{ntk}(\x_i,\x_{te})- \hat{\mathcal{I}}_{ntk}(\x_i,\x_{te})\right| \\\leq&
    \underbrace{\sqrt{\frac{\gamma}{n^2p_k}}|\mathcal{I}_{ntk}(\x_i,\x_{te})|}_{\text{{\uppercase\expandafter{\romannumeral1}}}} +   \underbrace{\frac{1}{2}\alpha(\x_i,\x_{te})^2(f_{ntk}^{\backslash i}(\x_i) - y_i)^2}_{\text{\uppercase\expandafter{\romannumeral2}}}.
\end{aligned}
\end{equation}
\end{theorem}

\noindent\textbf{Remark:}
 Notice that term $\text{\uppercase\expandafter{\romannumeral1}}$  decreases with $p_k$, and term $\text{\uppercase\expandafter{\romannumeral2}}$ represents the leave-one-out error which also decreases with $p_k$ in general. In the theorem, we require $\lambda>\frac{\sqrt{2} \lambda_{max}\epsilon_r}{1-\sqrt{2}\epsilon_r}$, which can be well controlled when $r$ is small, and our numerical experiments show the phenomenon is founded whether $\lambda$ is large or small.

\begin{figure}
\includegraphics[width=0.48\textwidth]{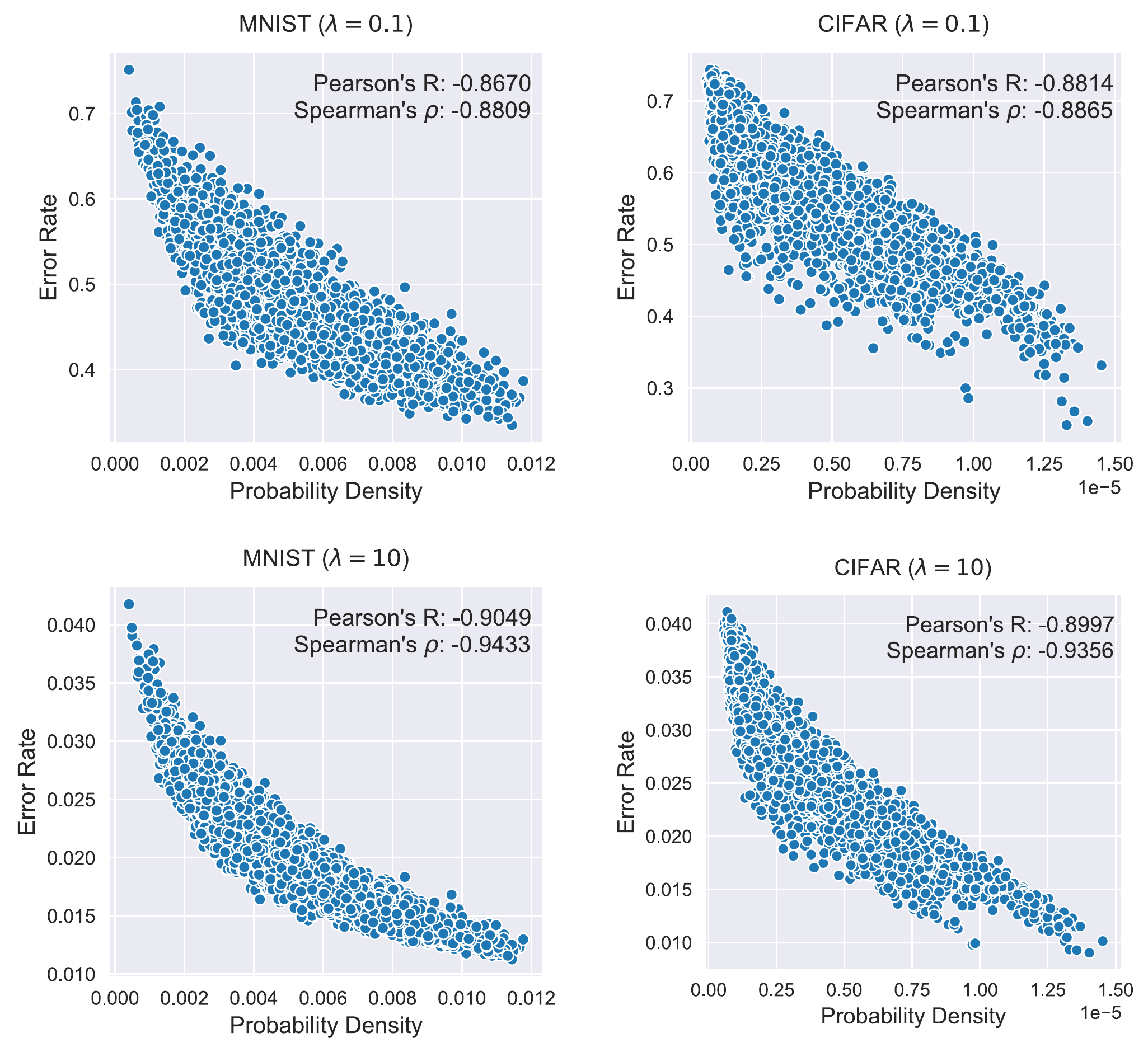}
\caption{The probability density of training data and its corresponding error rate of estimating the IFs via IHVP on MNIST and CIFAR, respectively. Density for training data is calculated through the Gaussian kernel density estimate method.} 
\label{pd} 
\end{figure}

We can see that the approximation error rate of the classic IHVP method has a significant correlation with the probability density of corresponding training points no matter $\lambda$ is large or not, which confirms our conclusion in the NTK regime (Figure \ref{pd}). More experiments on simulated data can be seen in the Appendix A.

\section{Towards Understanding of IFs}
\begin{figure}[h] \centering 
\includegraphics[width=0.48\textwidth]{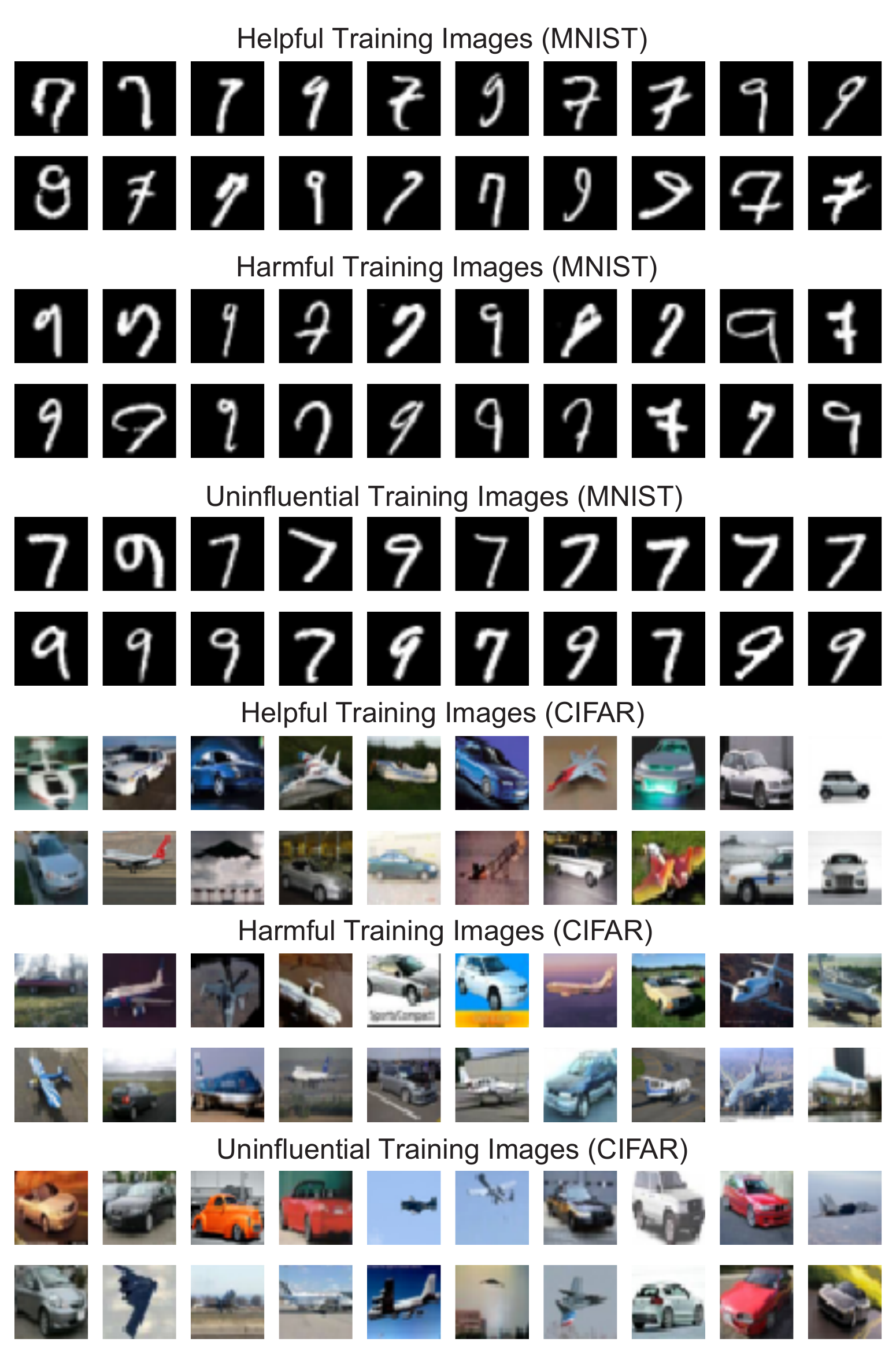}
\caption{The most helpful, harmful and uninfluential images in MNIST and CIFAR respectively.} 
\label{comp1} 
\end{figure}

The NTK theory can not only be utilized to calculate the IFs and estimate the approximation error, but also help to understand IFs better from two new views. On the one hand, we can quantify the complexity for the influential samples through the complexity metric in the NTK theory. On the other hand, we can trace the influential samples during the training process through the linear ODE depicting the training dynamics.

\subsection{Quantify the Complexity for Influential Samples}
After calculating the IF for every training point, one natural question is that what is the difference between influential samples and uninfluential ones. It is interesting to see in Figure \ref{comp1} that the uninfluential samples seem to be homogenized, while the most influential ones tend to be more complicated. For instance, planes in the uninfluential groups have the backgrounds of daylight or airport in general, while the backgrounds in the influential groups tend to be more diverse, such as grassland, dusk, midnight. To quantify this phenomenon, we borrow the complexity metric in NTK theory and define the complexity of ${\X}_{\mathcal{I}} \subset \X$ as follows:
\begin{equation}
    \mathcal{C}({\X}_{\mathcal{I}}) \triangleq \sqrt{{\Y^{\top}(\K^{\infty}_{tr})^{-1}\Y}} - \sqrt{{({\Y}^{\backslash \mathcal{I}})^{\top}({\K}^{\backslash \mathcal{I}}_{tr})^{-1}{\Y}^{\backslash \mathcal{I}}}},
\end{equation}
where ${\Y}^{\backslash \mathcal{I}}$ and $
{\K}_{tr}^{\backslash \mathcal{I}}$ denote the label sets and kernel matrix 
constituted without ${\X}_\mathcal{I}$ respectively. On the one hand, for a function $f(\x) = \sum_{i=1}^{n}\beta_{i}k(\x,\x_i)$ in the reproducing kernel Hilbert space (RKHS) $\mathbb{H}$, its RKHS norm is $\|f\|_{\mathbb{H}} = \sqrt{\boldsymbol{\beta}^{T}\K\boldsymbol{\beta}}$ \citep{mohri2018foundations}. Hence for $f_{ntk}(\x)$ in the NTK regime, we have $\|f_{ntk}\|_{\mathbb{H}} = \sqrt{{\Y^{\top}(\K^{\infty}_{tr})^{-1}\Y}}$, and $\mathcal{C}({\mathcal{X}}_{\mathcal{I}})$ actually denotes the increment of RKHS norm contributed from the training data ${\X}_{\mathcal{I}}$. On the other hand, $\sqrt{{\Y^{\top}(\K^{\infty}_{tr})^{-1}\Y}}$ controls the upper bound of Rademacher complexity for a class of neural networks, and thus controls the test error \citep{pmlr-v97-arora19a}.

Next, we explore the relationship between the complexity and the IF for each training point. In detail, we divide the training set into ten groups sorted by their IFs, from the most harmful groups to the most helpful groups. Then we calculate $\mathcal{C}(\X_{\mathcal{I}})$ for each group, and it is interesting to see that the more influential data make the model more complicated. Two most influential groups (Group 0 and 9) contribute the most to the model complexity (Figure \ref{comp2}). The results indicate that the most helpful data increase the complexity and the generalization ability of the model. In contrast, the most harmful data increase the complexity of the model but hurts the generalization.

\begin{figure}[t]
	\includegraphics[width=0.5\textwidth]{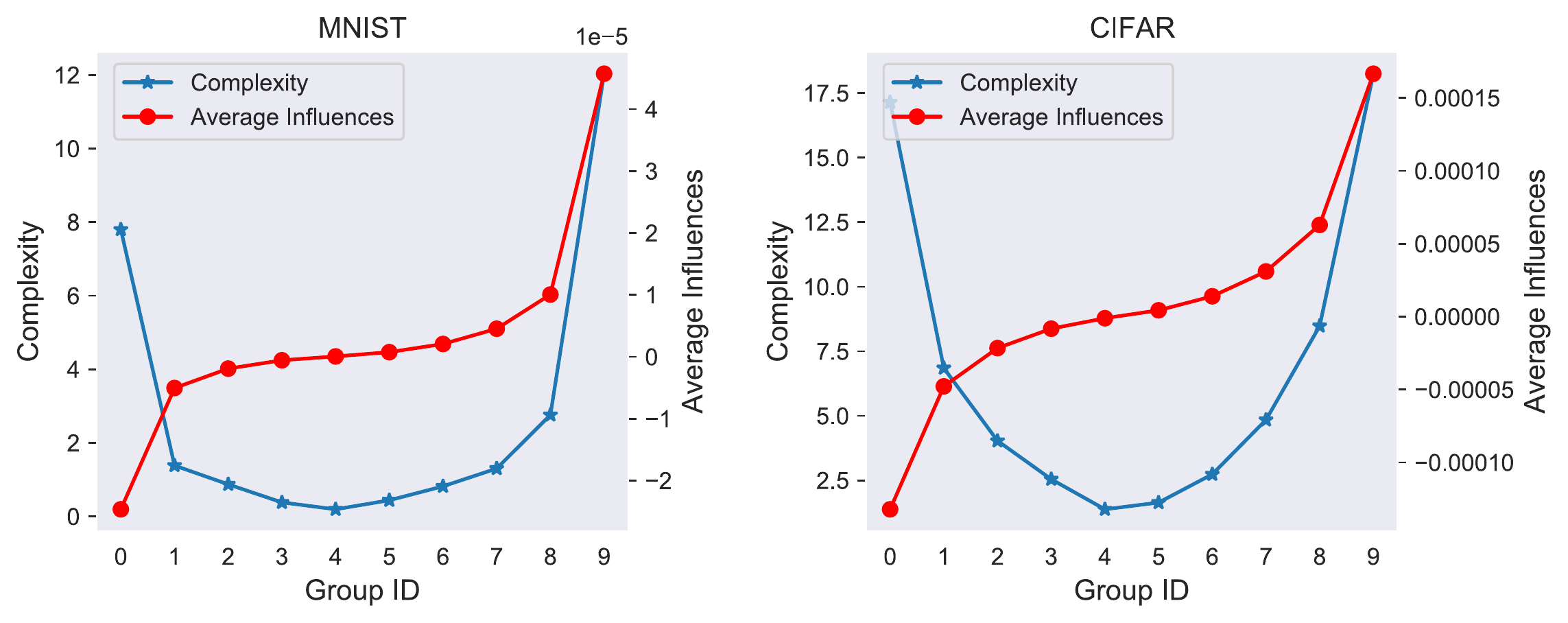}
	\caption{The relationship between complexity and influence of each group divided by their IFs on MNIST and CIFAR respectively. Notice that Group 0 and 9 denote the most harmful and helpful groups respectively, which contribute the most to the model complexity.} 
	\label{comp2} 
\end{figure}

\begin{figure}[!h]
	\includegraphics[width=0.49\textwidth]{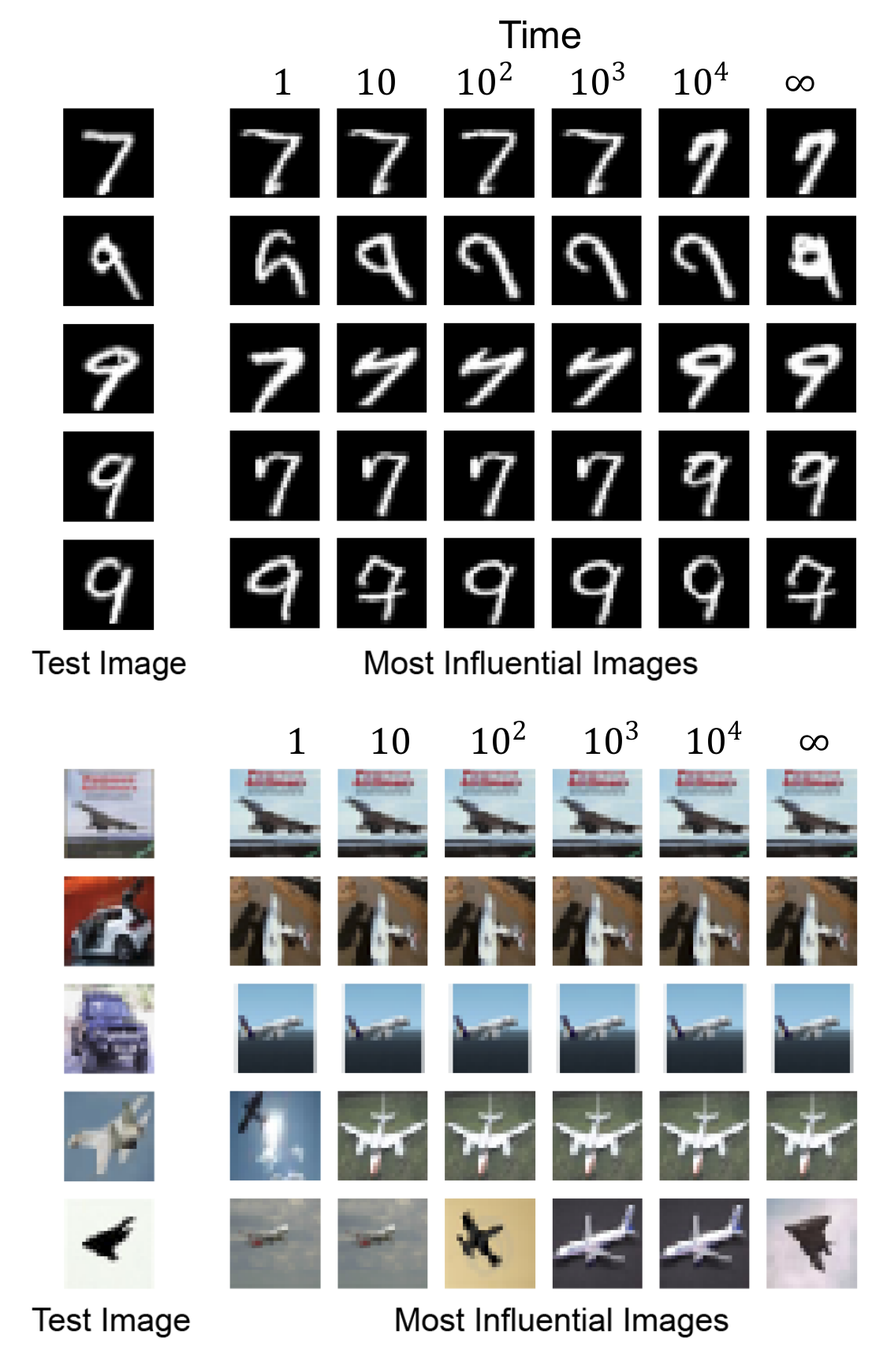}
	\caption{Variation of the most influential images for some test data on MNIST and CIFAR during the training dynamics.} 
	\label{dynamic1} 
\end{figure}

\subsection{Track Influential Samples During Training}
Different with IHVP, which can only calculate the IF at the end of training, the NTK theory depicts the training dynamics via the following linear ODE:
\begin{equation}
    f_{ntk}(\x_{te};t)=(\K^{\infty}_{te})^{\top}\left(\K^{\infty}_{tr} \right)^{-1}\left(\mathbf{I}_{n}-\exp\left({{-\frac{2 t}{n} \K^{\infty}_{tr}}}\right)\right) \Y,
\end{equation}
hence we can track influential samples efficiently during the training process. The most influential images for certain test points do not keep constant in general (Figure \ref{dynamic1}). Considering the variation of IFs in the different processes can help to understand the model behavior better. For instance, we track the most influential images for every test point in the presence of label noise during the training process and record the proportion of the clean and noise data in the most influential data in Figure \ref{dynamic2} respectively.

At the beginning of the training process, most test images are affected by the clean samples (Figure \ref{dynamic2}). However, as the training progresses, the influence of noise samples begin to dominate the training, which means that the model begins to learn the noise data. After that, the model has learned the label noise, thus the influence of the noise samples begin to decrease gradually. Numerical experiments reveal the variation of the most influential samples in the presence of label noise and help us understand why early-stopping can prevent over-fitting from the perspective of IFs (Figure \ref{dynamic2}).

\begin{figure}[htb]
\includegraphics[width=0.49\textwidth]{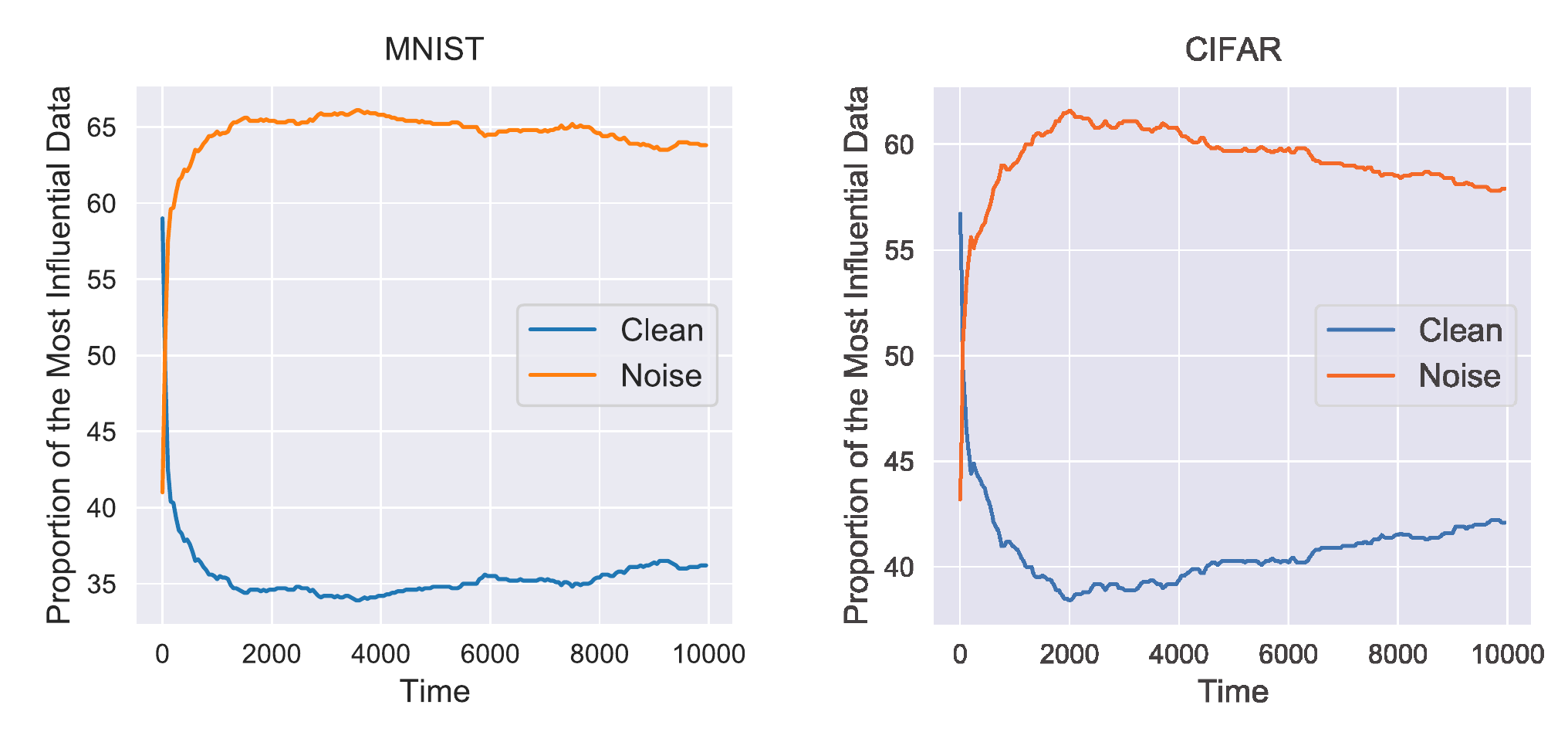}
\caption{The proportion of the most influential images for 1000 test images on MNIST and CIFAR with the presence of label noise during the training dynamics. We flip the labels of 40\% training data to simulate the label noise.} 
\label{dynamic2} 
\end{figure}

\section{Conclusion}
In this paper, we calculate the IFs for over-parameterized neural networks and utilize the NTK theory to understand when and why the IHVP method fails or not. At last, we quantify the complexity and track the variation of influential samples. Our research can help understand IFs better in the regime of neural networks and bring new insights for explaining the training process through IFs. Some future works are mentioned as follows:
(1) Explore the approximation error caused by stochastic estimation in IHVP; (2) Generalize our results in Theorem \ref{theorem2} to broad scenarios, such as deep neural networks, convolutional neural networks and non-convex loss functions; (3) Track the IFs to reveal what does neural networks learn during training dynamics.

\clearpage
\section{Acknowledgments}
We thank reviewers for their helpful advice. Rui Zhang would like to thank Xingbo Du for his valuable and detailed feedback to improve the writing. This work has been supported by the National Key Research and Development Program of China [2019YFA0709501]; the Strategic Priority Research Program of the Chinese Academy of Sciences (CAS) [XDA16021400, XDPB17]; the National Natural Science Foundation of China [61621003]; the National Ten Thousand Talent Program for Young Top-notch Talents. 

\bibliography{aaai22}

\onecolumn
\begin{center}
   \LARGE{\textbf{Appendix: Rethinking Influence Functions of Neural Networks\\ in the Over-parameterized Regime}}
\end{center}

\section{Appendix A: Experimental Settings and Supplementary Experiments}
\subsection{Experimental Settings}
The architecture of neural networks in our experiments is described in the main text. We train the neural networks using full-batch gradient descent, with a fixed learning rate $lr = 0.001$. Theorem \ref{theorem2} requires $\kappa$ to be small in the initialization, thus we fix $\kappa = 0.01$ in the experiments. We set the number of epochs as $5000$ to ensure the training loss converges. We only use the first two classes of images in the MNIST (`7' v.s. `9') \cite{726791} and CIFAR (`car' v.s. `plane') \cite{krizhevsky2009learning}, with 5000 training images and 1000 test images in total for each dataset. In both datasets, we normalize each image $\x_i$ such that $\|\x_i\|_2=1$ and set the corresponding label $y_i$ to be $+1$ for the first class, and $-1$ otherwise. The neural networks are trained via the PyTorch package \cite{NEURIPS2019_bdbca288}, using NVIDIA GeForce RTX 3090 24GB.
\subsection{Supplementary Experiments}
This section repeats the experiments for multi-class classification tasks on MINI MINST and MINI CIFAR, respectively. All experimental settings are the same as the previous section, and MINI MINST and MINI CIFAR are constructed by 5000 training images and 1000 test images from MNIST and CIFAR, respectively. Numerical experiments verify our theory on multi-class classification task (Figure \ref{supp1}, \ref{supp2}, \ref{supp3} and \ref{supp4}). Furthermore, we verify the upper bound in Theorem 4 on simulated data (Figure \ref{supp5}). 

\begin{figure}[H]\centering
\includegraphics[width=0.6\textwidth]{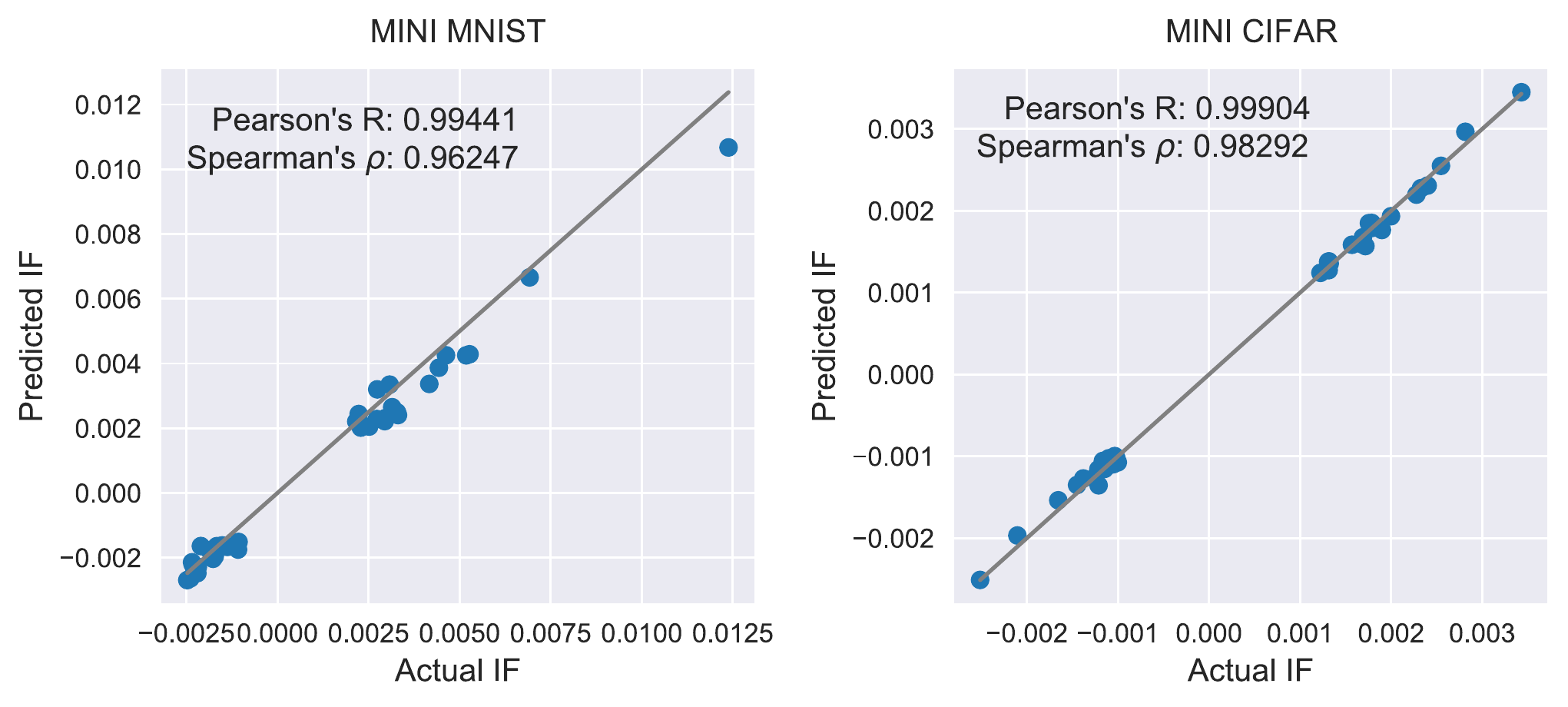} 
\caption{\textbf{{Comparison of predicted IFs and actual ones on the 40 most influential training points for a two-layer ReLU networks with $8 \times 10^4$ neurons in the hidden layer. }}
We set $\lambda = 4$ on MINI MNIST and $\lambda=10$ on MINI CIFAR respectively.} 
\label{supp1}\end{figure}

\begin{figure}[H]\centering
\includegraphics[width=0.6\textwidth]{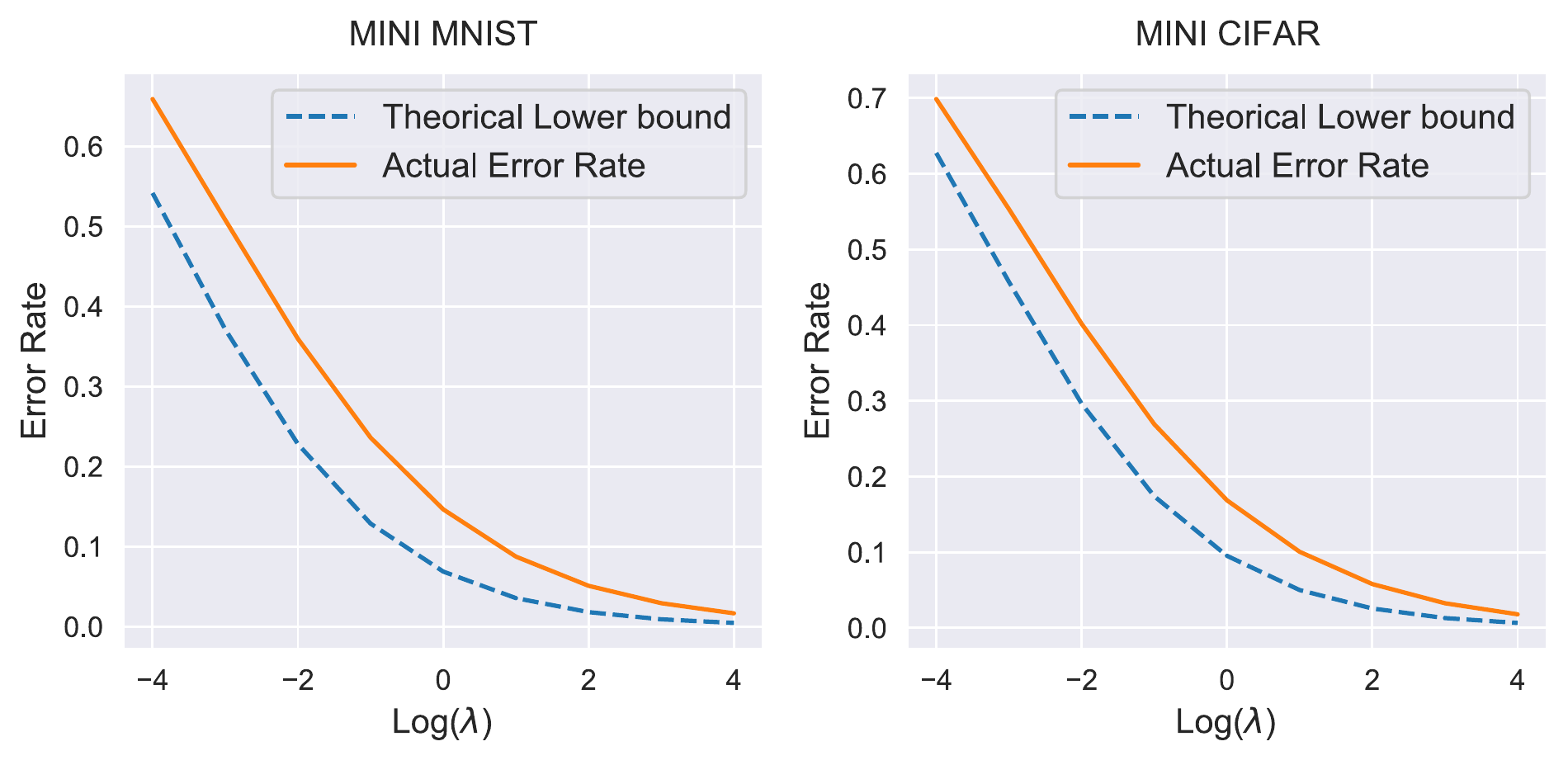} 
\caption{\textbf{{The Effects of Regularization.}} Comparison of the theoretical lower bound of error rates and the actual error rates for different $\lambda$s on MINI MNIST and MINI CIFAR respectively.} 
\label{supp2}\end{figure}

\begin{figure}[h]\centering
\includegraphics[width=0.48\textwidth]{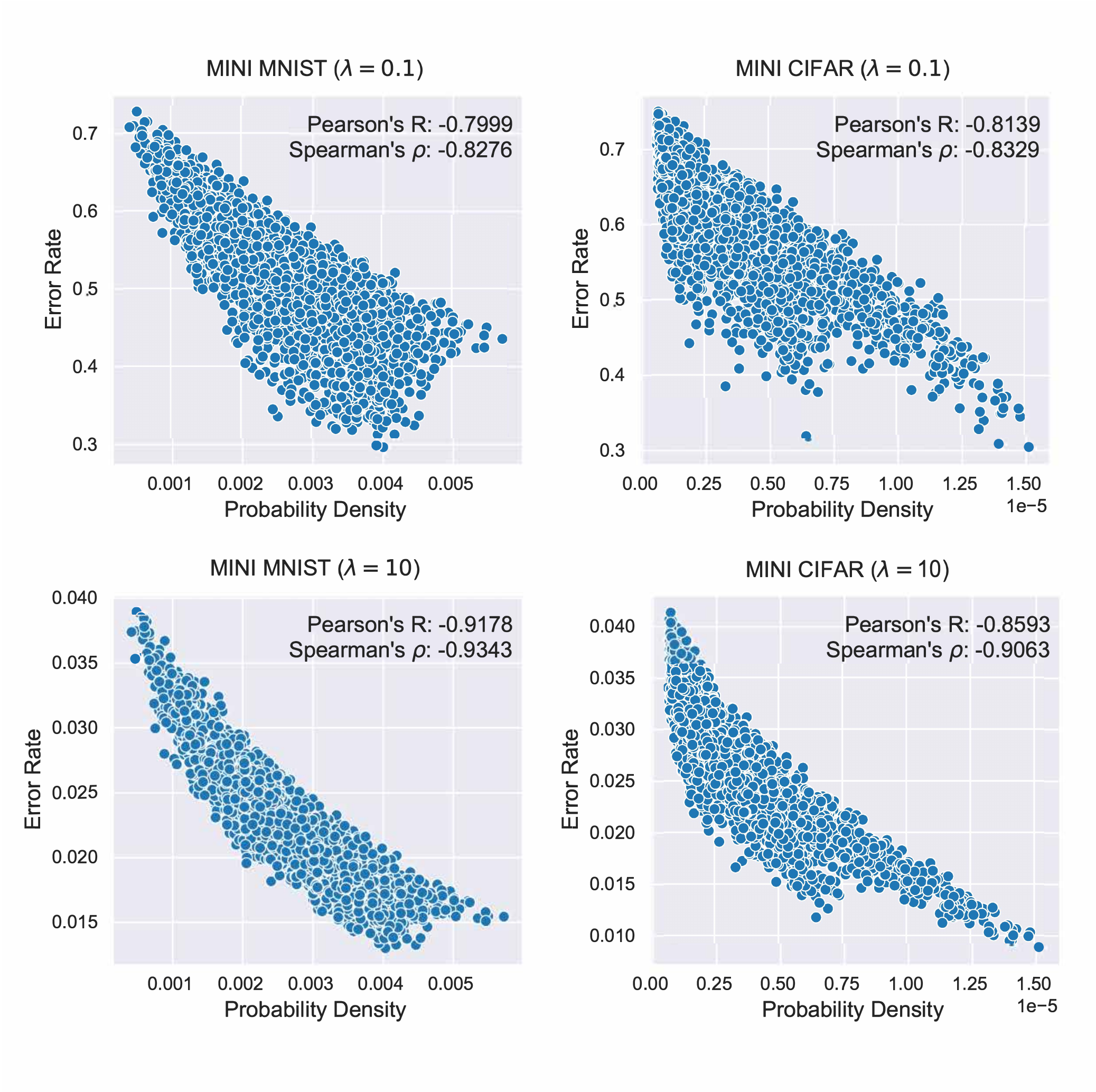}
\caption{\textbf{{The Effects of Training Points.}} Probability density of training data and its corresponding error rates when estimates the IFs via the classic IHVP method on MINI MNIST and MINI CIFAR respectively. Density for training data is calculated through the Gaussian kernel density estimate method.} \label{supp3}
\end{figure}

\begin{figure}[h]\centering
\includegraphics[width=0.53\textwidth]{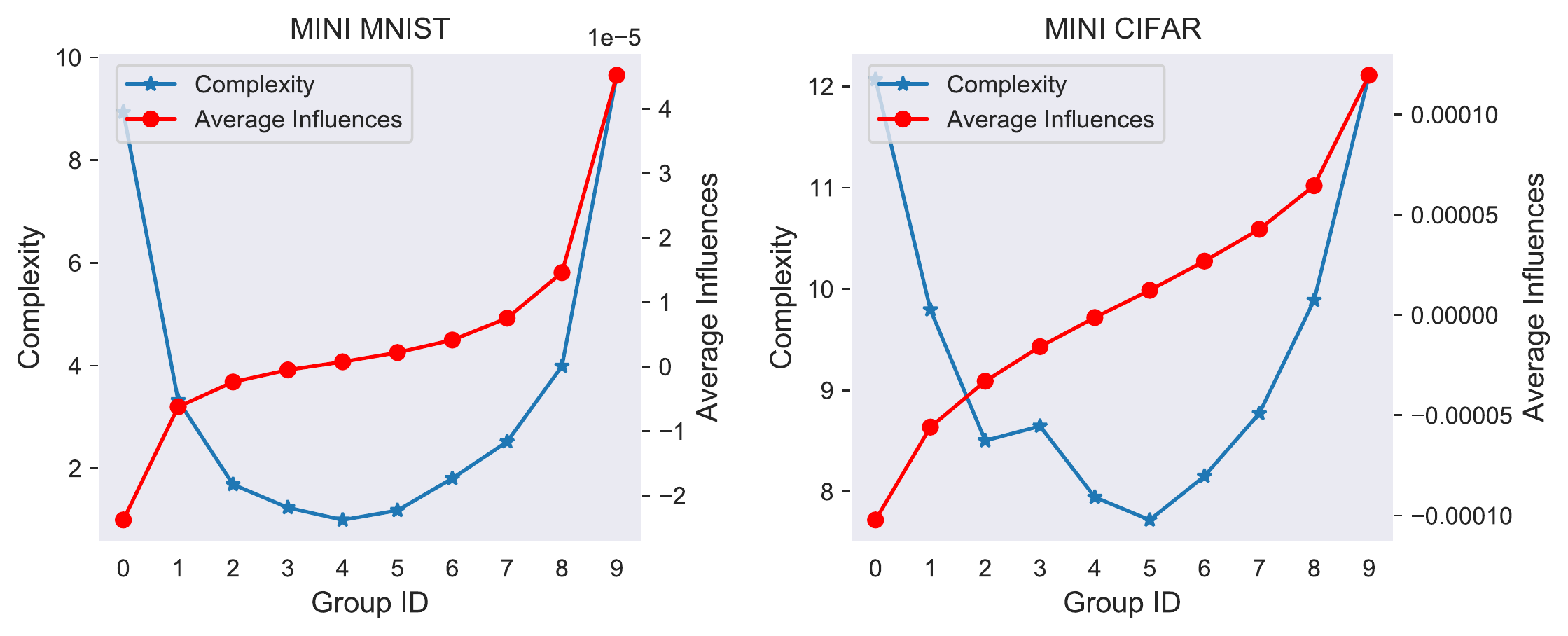}
\caption{\textbf{Quantify the Complexity for Influential Samples.} The relationship between complexity and influence of each group divided by their IFs on MINI MNIST and MINI CIFAR respectively.} \label{supp4}
\end{figure}

\begin{figure}[h]\centering
\includegraphics[width=0.25\textwidth]{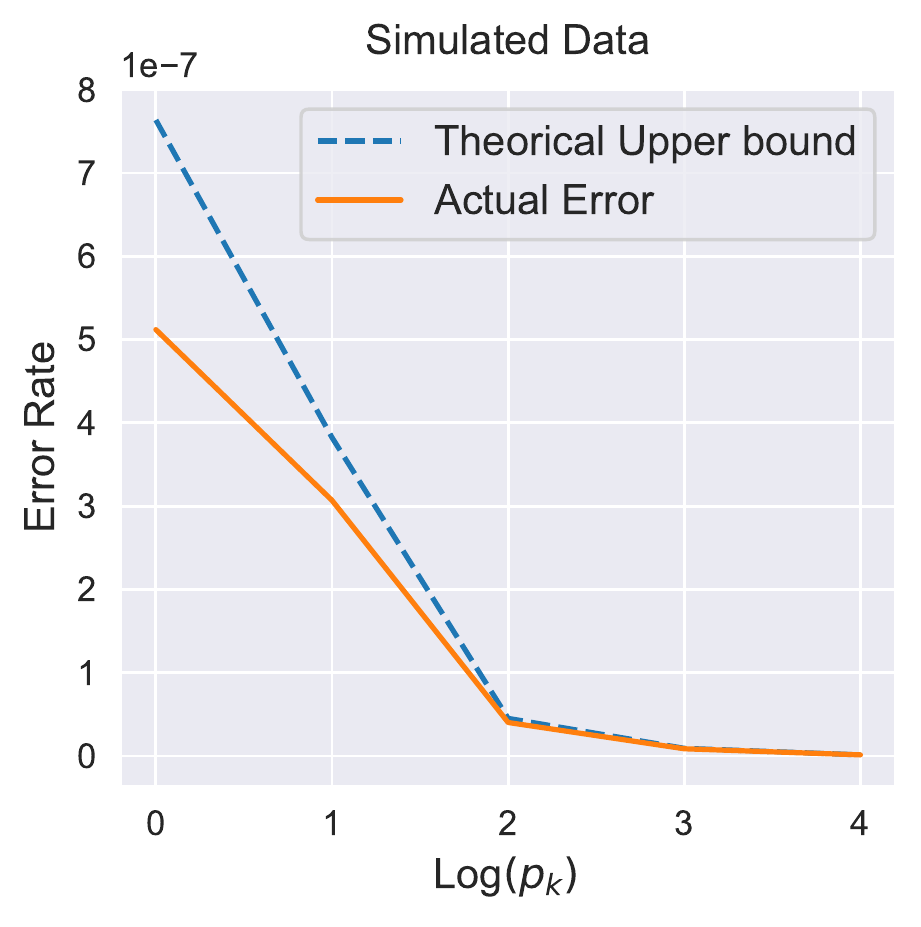}
\caption{\textbf{Comparison of the theoretical upper bound of error and the actual error in Theorem 4.} We generate the positive samples from $5$ clusters in $\R^{100}$ and set $n_k=50\times 2^k$, and the same for negative samples. We set $r=0.001$, and other parameters are calculated based on Theorem 4.} 
\label{supp5}
\end{figure}

\clearpage
\section{Appendix B: Proofs for Section 5}
In this paper, we consider minimizing the loss function $\Loss(\W)$ in the gradient flow regime, i.e., gradient descent with infinitesimal step size. The evolution of $\W$ can be described by the following ordinary differential equation (ODE):

\begin{equation}
\begin{aligned}
    \frac{\df\W(t)}{\df t} &=- \frac{\partial \Loss(\W(t))}{\partial \W(t)} \\ 
    &=  
    \sum_{i=1}^{n} (y_i - f_{nn}(\x_i;\W(t)))\frac{\partial f_{nn}(\x_i;\W(t))}{\partial \W(t)} - \lambda (\W(t)- \W(0)),
\end{aligned}
\end{equation}
where $\W(t)$ denotes the parameters of the neural network at time $t$. According to the chain rule, the evolution of $f_{nn}(\X)$ can be derived as follows:

\begin{equation}
\begin{aligned}
    \frac{\df f_{nn}(\X;\W(t))}{\df t} &= 
    \left\langle \frac{\partial f_{nn}(\X;\W(t))}{\partial \W(t)}, \frac{\df \W(t)}{\df t} \right\rangle
    \\&=-\K_{tr}(t)(f_{nn}(\X;\W(t)) - \Y) - \lambda \J(\X;t)^{\top}\left(\operatorname{vec}(\W(t)) - \operatorname{vec}(\W(0)) \right) ,
\end{aligned}
\end{equation}
where $\J(\X;t) = [\J(\x_1;t),\cdots,\J(\x_n;t)]\triangleq \left[\operatorname{vec}\left(\frac{\partial f_{nn}(\x_{1};\W(t))}{\partial \W(t)}\right) , \cdots , \operatorname{vec}\left(\frac{\partial f_{nn}(\x_{n}, \W(t)}{\partial \W(t)}\right)\right] \in \R^{md \times n}$ is the Jacobian matrix of the neural network, and $\K_{tr}(t) \triangleq \J(\X;t)^{\top}\J(\X;t) $ is an $n\times n$ positive semi-definite matrix. Notice that $\sigma(x) = x \I(x>0) = x\sigma^{\prime} (x)$ for all $x \in \R$, hence we can rewrite the evolution of $f_{nn}(\X;\W(t))$ as follows:
\begin{equation}
    \frac{\df f_{nn}(\X;\W(t))}{\df t}+(\K_{tr}(t)+\lambda \mathbf{I}_n)(f_{nn}(\X;\W(t))-\Y) = -\lambda \Y + \lambda \J(\X;t)^{\top} \operatorname{vec}({\W(0)})\label{X_ode_new}.
\end{equation}

For an arbitrary test point $\x_{te}$, we have:
\begin{equation}
    \frac{\df f_{nn}(\x_{te};\W(t))}{\df t} + \lambda f_{nn}(\x_{te};\W(t))=-\mathbf{K}_{te}(t)^{\top}(f_{nn}(\X;\W(t))-\Y)+\lambda \J(\x_{te};t)^{\top} \operatorname{vec}({\W(0)}) , 
\end{equation}
where $\J(\x_{te};t) \triangleq \operatorname{vec}\left(\frac{\partial f_{nn}(\x_{te};\W(t))}{\partial \W(t)}\right) \in \R^{md}$ and $\mathbf{K}_{te}(t) \triangleq \J(\X;t)^{\top}\J(\x_{te};t) \in \R^{n}$ respectively. As $t \to \infty$, we define $f_{nn}(\x_{te}) \triangleq \lim_{t\to \infty} f_{nn}(\x_{te};\W(t))$ as the prediction of the neural network at the end of training.

To prove Theorem \ref{theorem1}, we first analyze the evolution of $\W(t)$, and present the following lemma to reveal that the perturbation of each $\mathbf{w}_r$ only moves little when the width is sufficiently large.
\begin{lemma}\label{lemma1} 
With probability at least $1 - \delta$ over random initialization, we have $\|\w_r(t)-\w_r(0)\|_2 =  \mathcal{O}\left(\frac{n}{\sqrt{\delta m}\lambda}\right)$ for all $t \ge 0$ and $r \in [m]$.
\end{lemma}
\begin{proof}
Because the loss function $\mathcal{L}(\W(t))$ is non-increasing during the training dynamics, thus we have:
\begin{equation}
    \frac{1}{2}\sum_{i=1}^{n}(f_{nn}(\x_i;\W(t))-y_i)^2
    \leq \mathcal{L}(\W(t))\leq \mathcal{L}(\W(0)).
\end{equation}
On the other hand, we have:
\begin{equation}
\begin{aligned}
    \mathbb{E}\left[\mathcal{L}(\W(0))\right]&=\mathbb{E}\left[\frac{1}{2}\sum_{i=1}^{n}(f_{nn}(\x_i;\W(t))-y_i)^2+\frac{\lambda}{2}\|\W(0)-\W(0)\|_F^2\right]\\&= 
    \frac{1}{2}\sum_{i=1}^{n}\left(y_i^2+\mathbb{E}[f_{nn}(\x_i;\W(0))^2]-2y_i\mathbb{E}[f_{nn}(\x_i;\W(0))]\right)
    \\&=\mathcal{O}(n).
\end{aligned}
\end{equation}
According to the Markov’s inequality, with probability at least $1-\delta$, we have:
\begin{equation}
    \frac{1}{2}\sum_{i=1}^{n}(f_{nn}(\x_i;\W(t))-y_i)^2
    \leq \mathcal{L}(\W(0)) = \mathcal{O}\left(\frac{n}{\delta}\right).
\end{equation}
Now we consider the training dynamics of $\mathbf{w}_r(t)$ for any $r \in [m]$, which is controlled by the ODE as follows:
\begin{equation}
    \begin{aligned}
        \frac{\df \mathbf{w}_r(t)}{\df t} =& -\frac{\partial \Loss (\W(t))}{\partial \mathbf{w}_r(t)}\\ 
        = & -\sum_{i=1}^{n}(f_{nn}(\x_i;\W(t))-y_i)\frac{\partial f_{nn}(\x_i;\W(t))}{\partial \mathbf{w}_{r}(t)} - \lambda (\mathbf{w}_r(t)- \mathbf{w}_r(0))\\=&-\mathbf{Z}_r(t)(f_{nn}(\X;\W(t))-\Y) - \lambda (\mathbf{w}_r(t)-\mathbf{w}_r(0)),
    \end{aligned}
\end{equation}
where $\mathbf{Z}_r(t) \triangleq \left[\frac{\partial f_{nn}(\x_{1};\mathbf{W}(t))}{\partial \mathbf{w}_r(t)} , \cdots , \frac{\partial f_{nn}(\x_{n}, \mathbf{W}(t))}{\partial \w_{r}(t)}\right] \in \R^{d \times n}$, and we can bound $\|\mathbf{Z}_r(t)\|_F^2$ for all $t \ge 0$ as follows:
\begin{equation}
\|\mathbf{Z}_r(t)\|_F^2 = \sum_{i=1}^{n}   \frac{1}{m}a_r^2 \mathbb{I}\{\mathbf{w}_r(t)^{\top}\x_i\geq 0\}\|\x_i\|_2^2 \leq \frac{n}{m}. 
\end{equation}
Notice that the training dynamic of $\mathbf{w}_r(t)$ satisfies:
\begin{equation}
  \frac{\df (\mathbf{w}_r(t) - \mathbf{w}_r(0))}{\df t} 
   +\lambda (\mathbf{w}_r(t) - \mathbf{w}_r(0)) = -\mathbf{Z}_r(t)(f_{nn}(\X;\W(t))-\Y),
\end{equation}
hence we can bound $\|\w_r(t)-\w_r(0)\|_2$ for all $t \ge 0$ and $r \in [m]$ as follows:
\begin{equation}
\begin{aligned}
        \|\w_r(t)-\w_r(0)\|_2 &=\left\|-e^{-\lambda t}\int_0^t e^{\lambda s}\mathbf{Z}_r(s)(f_{nn}(\X;\W(s))-\Y)\df s\right\|_2\\
        &\leq \max_{0 \leq s \leq t}\left\|\mathbf{Z}_r(s)(f_{nn}(\X;\W(s))-\Y)\right\|_2\cdot \left|e^{-\lambda t}\int_0^t e^{\lambda s}\df s\right|\\
        &\leq \max_{0 \leq s \leq t}\frac{\|\mathbf{Z}_r(s)\|_F \|f_{nn}(\X;\W(s))-\Y\|_2}{\lambda}\\
        &= \mathcal{O}\left(\frac{n}{\sqrt{\delta m}\lambda}\right).
\end{aligned}
\end{equation}
\end{proof}

Let $R$ be the maximum perturbation for all $\mathbf{w}_r$ during the training process, i.e., $R\triangleq \max_{t>0}\|\mathbf{w}_r(t)-\mathbf{w}_r(0)\|_2$. According to Lemma \ref{lemma1}, we have $R = \mathcal{O}\left(\frac{n}{\sqrt{\delta m}\lambda}\right)$ with probability at least $1 - \delta$. Next, we prove the following important terms can be controlled by $R$, which plays a crucial role in the proof of Theorem \ref{theorem1}.

\begin{lemma}\label{lemma2}
If $\mathbf{w}_{1}(0), \cdots, \mathbf{w}_{m}(0)$ are i.i.d. generated from $\mathcal{N}(\mathbf{0}, \kappa^2\mathbf{I}_d)$, then the following holds with probability at least $1-\delta$. For any set of weight vectors $\mathbf{w}_{1}, \cdots, \mathbf{w}_{m} \in \mathbb{R}^{d}$ that satisfy for any $r \in[\mathrm{m}]$, $\left\|\mathbf{w}_{r}(0)-\mathbf{w}_{r}\right\|_{2} \leq  R$, then the Gram matrix $\mathbf{K}_{tr} \in \mathbb{R}^{n \times n}$, $\mathbf{K}_{te} \in \R^{n}$ and the Jacobian matrix $\mathbf{J}(\X) \in \R^{md \times n}$, $\mathbf{J}(\x_{te}) \in \R^{md}$ defined by
\begin{equation}
    \begin{aligned}
	\left[\mathbf{K}_{tr}\right]_{ij}\triangleq \frac{1}{m} \mathbf{x}_{i}^{\top} \mathbf{x}_{j} \sum_{r=1}^{m} \mathbb{I}\left\{\mathbf{w}_{r}^{\top} \mathbf{x}_{i} \geq 0, \mathbf{w}_{r}^{\top} \mathbf{x}_{j} \geq 0\right\}&;\\ 
[\mathbf{K}_{te}]_{i} \triangleq \frac{1}{m} \mathbf{x}_{i}^{\top} \mathbf{x}_{te} \sum_{r=1}^{m} \mathbb{I}\left\{\mathbf{w}_{r}^{\top} \mathbf{x}_{i} \geq 0, \mathbf{w}_{r}^{\top} \mathbf{x}_{te} \geq 0\right\}&;\\
\J(\X) \triangleq \left[\operatorname{vec}\left(\frac{\partial f_{nn}(\x_{1};\W)}{\partial \W}\right) , \cdots , \operatorname{vec}\left(\frac{\partial f_{nn}(\x_{n}, \W)}{\partial \W}\right)\right]; \quad \J(\x_{te}) \triangleq& \operatorname{vec}\left(\frac{\partial f_{nn}(\x_{te};\W)}{\partial \W}\right),
    \end{aligned}
\end{equation}
satisfying: 
\begin{itemize}
    \item $\|\K_{tr} - \K_{tr}(0)\|_2 \leq \frac{4n^2R}{\sqrt{2 \pi}\kappa\delta}$;
    \item $\left\|(\J(\X) - \J(\X;0))^{\top}\operatorname{vec}(\W(0))\right\|_2 \leq \frac{2\sqrt{m}nR^2}{\sqrt{2\pi}\kappa \delta}$;
    \item $\|\K_{te} - \K_{te}(0)\|_2 \leq \frac{4nR}{\sqrt{2 \pi}\kappa\delta}$;
    \item $\left\|(\J(\x_{te}) - \J(\x_{te};0))^{\top}\operatorname{vec}(\W(0))\right\|_2 \leq \frac{2\sqrt{m}R^2}{\sqrt{2\pi}\kappa \delta}$.
\end{itemize}
\end{lemma}
\begin{proof}
Without loss of generality, we only proof the first two claims. Similarly with the proof of Lemma 3.2 in \citep{du2018gradient}, we define the event as follows:
\begin{equation}
    A_{i r}=\left\{\exists \mathbf{w}:\left\|\mathbf{w}-\mathbf{w}_{r}(0)\right\|_2 \leq R, \mathbb{I}\left\{\mathbf{x}_{i}^{\top} \mathbf{w}_{r}(0) \geq 0\right\} \neq \mathbb{I}\left\{\mathbf{x}_{i}^{\top} \mathbf{w} \geq 0\right\}\right\}.
\end{equation}
Note that the event $A_{ir}$ happens if and only if $\left|\mathbf{w}_{r}(0)^{\top} \mathbf{x}_{i}\right|<R$. Recall that $\mathbf{w}_{r}(0) \sim \mathcal{N}(\mathbf{0}, \kappa^2\mathbf{I}_d)$. By anti-concentration inequality of Gaussian, we have $\mathbb{P}\left[A_{i r}\right]=\mathbb{P}_{z \sim \mathcal{N}(0,\kappa^2)}\left[|z|<R\right] \leq \frac{2 R}{\sqrt{2 \pi}\kappa}$. According to the proof of Lemma 3.2 in \citep{du2018gradient}, with probability of $1-\delta$, we have:
\begin{equation}
    \|\K_{tr} - \K_{tr}(0)\|_2 \leq \frac{4n^2R}{\sqrt{2 \pi}\kappa\delta}.
\end{equation}

Now we bound $\mathbb{E}\left[\|(\J({\X}) - \J(\X;0))^{\top}\operatorname{vec}(\W(0))\|_2\right]$ as follows:
\begin{equation}
\begin{aligned}
&\mathbb{E}\left[\left\|(\J(\X) - \J(\X;0))^{\top}\operatorname{vec}(\W(0))\right\|_2\right]
\\ \leq & \frac{1}{\sqrt{m}}\sum_{r=1}^{m}\sum_{i=1}^{n} \mathbb{E}\left[\left|a_r(\mathbb{I}(\mathbf{w}_{r}^{\top}\x_i\geq0)-\mathbb{I}(\mathbf{w}_{r}(0)^{\top}\x_i\geq0))\mathbf{w}_{r}(0)^{\top}\x_i)\right|\right]
\\ \leq & \frac{1}{\sqrt{m}}\sum_{r=1}^{m}\sum_{i=1}^{n} \mathbb{E}\left[\left|\mathbb{I}(\mathbf{w}_{r}^{\top}\x_i\geq0)-\mathbb{I}(\mathbf{w}_{r}(0)^{\top}\x_i\geq0)\right|\cdot\left|\mathbf{w}_{r}(0)^{\top}\x_i\right|\right]
\\ \leq & \frac{1}{\sqrt{m}}\sum_{r=1}^{m}\sum_{i=1}^{n} \mathbb{E}\left[\left|\mathbb{I}(\mathbf{w}_{r}^{\top}\x_i\geq0)-\mathbb{I}(\mathbf{w}_{r}(0)^{\top}\x_i\geq0)\right|\cdot\left|\mathbf{w}_{r}(0)^{\top}\x_i\right|\bigg|\left|\mathbf{w}_{r}(0)^{\top}\x_i\right|<R\right]\mathbb{P}\left[\left|\mathbf{w}_{r}(0)^{\top}\x_i\right|<R\right]+\\
&  \frac{1}{\sqrt{m}}\sum_{r=1}^{m}\sum_{i=1}^{n} \mathbb{E}\left[\left|\mathbb{I}(\mathbf{w}_{r}^{\top}\x_i\geq0)-\mathbb{I}(\mathbf{w}_{r}(0)^{\top}\x_i\geq0)\right|\cdot\left|\mathbf{w}_{r}(0)^{\top}\x_i\right|\bigg|\left|\mathbf{w}_{r}(0)^{\top}\x_i\right|\ge R\right]\mathbb{P}\left[\left|\mathbf{w}_{r}(0)^{\top}\x_i\right|\ge R\right]\\
\leq & \frac{1}{\sqrt{m}}\sum_{r=1}^{m}\sum_{i=1}^{n} \left( R \cdot \mathbb{P}\left[\left|\mathbf{w}_{r}(0)^{\top}\x_i\right|<R\right]+0\cdot \mathbb{P}\left[\left|\mathbf{w}_{r}(0)^{\top}\x_i\right|\ge R\right]\right)\\
\leq & \frac{R}{\sqrt{m}}\sum_{r=1}^{m}\sum_{i=1}^{n}\mathbb{P}\left[A_{ir}\right]\\
\leq & \frac{2\sqrt{m}nR^2}{\sqrt{2\pi}\kappa}.
\end{aligned}
\end{equation}
According to the Markov’s inequality, with probability at least $1-\delta$, we have $\|(\J_{tr} - \J_{tr}(0))^{\top}\operatorname{vec}(\W(0))\|_2 \leq \frac{2\sqrt{m}nR^2}{\sqrt{2\pi}\kappa \delta}$.
\end{proof}

Now we consider the linearization of $f_{nn}(\x,\W)$. Specifically, we replace the neural network by its first order Taylor expansion at initialization:
\begin{equation}
    f_{lin}(\x;\W) \triangleq f_{nn}(\x;\W(0))+\left\langle\frac{\partial f_{nn}(\x;\W(0))}{\partial \W(0)}, \W-\W(0)\right\rangle.
\end{equation}
Equipped with the above lemmas, we bound the deviation of $f_{nn}(\x;\W(t))$ and $f_{lin}(\x;\W(t))$ on $\X$ and $\x_{te}$ respectively via the following two lemmas.
\begin{lemma}\label{lemma3}
With probability at least $1 - \delta$ over random initialization, for all $t>0$, we have:
\begin{equation}
\|f_{nn}(\X;\W(t)) - f_{lin}(\X;\W(t))\|_2 =\mathcal{O}\left(\frac{n^{\frac{7}{2}}}{\lambda^2 \delta^2\kappa \sqrt{m}}\right).
\end{equation}
\end{lemma}
\begin{proof}
Consider the dynamics of $f_{nn}(\X;\W(t))$ and $f_{lin}(\X;\W(t))$, we have:
\begin{equation}
    \begin{aligned}
&\frac{\df f_{nn}(\X;\W(t))}{\df t} + (\K_{tr}(t)+\lambda \mathbf{I}_n)f_{nn}(\X;\W(t)) = \K_{tr}(t)\Y + \lambda \J(\X;t)^{\top}\operatorname{vec}(\W(0));\\
&\frac{\df f_{lin}(\X;\W(t))}{\df t} + (\K_{tr}(0)+\lambda \mathbf{I}_n)f_{lin}(\X;\W(t)) = \K_{tr}(0)\Y + \lambda \J(\X;0)^{\top}\operatorname{vec}(\W(0)).
    \end{aligned}
\end{equation}
Combining the two ODE systems, we obtain:
\begin{equation}
    \begin{aligned}
    &\frac{\df}{\df t}\left(f_{nn}(\X;\W(t)) - f_{lin}(\X;\W(t))\right) +
    (\K_{tr}(0)+\lambda \mathbf{I}_n)\left(f_{nn}(\X;\W(t)) - f_{lin}(\X;\W(t))\right) \\
    =&(\K_{tr}(0) - \K_{tr}(t))\left(f_{nn}(\X;\W(t)) - \Y\right)
    + \lambda (\J(\X;t) - \J(\X;0))^{\top}\operatorname{vec}(\W(0)).
    \end{aligned}
\end{equation}
Then with probability at least $1 - \delta$, we can bound $\|f_{nn}(\X;\W(t)) - f_{lin}(\X;\W(t))\|_2$ for all $t>0$ as follows:
\begin{equation}
    \begin{aligned}
    &\|f_{nn}(\X;\W(t)) - f_{lin}(\X;\W(t))\|_2 \\
    \leq&\max_{0<s<t}\left(\frac{\|\K_{tr}(s) - \K_{tr}(0)\|_2 \|f_{nn}(\X;\W(s))-\Y\|_2 }{\lambda} + \left\|(\J_{tr}(s) - \J_{tr}(0))^{\top}\operatorname{vec}(\W(0))\right\|_2\right)\\
    = & \mathcal{O}\left(\frac{n^{\frac{5}{2}}R}{\lambda \delta^{\frac{3}{2}}\kappa}\right) +\mathcal{O}\left(\frac{\sqrt{m}nR^2}{\delta\kappa }\right)\\
    = & \mathcal{O}\left(\frac{n^{\frac{7}{2}}}{\lambda^2 \delta^2\kappa \sqrt{m}}\right).
    \end{aligned}
\end{equation}
\end{proof}

\begin{lemma} \label{lemma4}
With probability at least $1-\delta$ over the random initialization, for all $t>0$ and for any $\x_{te} \in \mathbb{R}^{d}$ with $\|\x_{t e}\|_2=1$, we have:
\begin{equation}
    |f_{nn}(\x_{te};\W(t)) - f_{lin}(\x_{te};\W(t)| = \mathcal{O}\left(\frac{n^4}{\lambda^3 \delta^2 \kappa \sqrt{m}}\right).
\end{equation}
\end{lemma}

\begin{proof}
Consider the dynamic systems of $f_{nn}(\x_{te};\W(t))$ and $f_{lin}(\x_{te};\W(t))$, we have:
\begin{equation}
    \begin{aligned}
&\frac{\df f_{nn}(\x_{te};\W(t))}{\df t} + \lambda f_{nn}(\x_{te};\W(t))= -\K_{te}(t)^{\top}\left(f_{nn}(\X;\W(t))-\Y\right) + \lambda \J(\x_{te};t)^{\top}\operatorname{vec}(\W(0));\\
&\frac{\df f_{lin}(\x_{te};\W(t))}{\df t} + \lambda f_{lin}(\x_{te};\W(t))= -\K_{te}(0)^{\top}\left(f_{lin}(\X;\W(t))-\Y\right) + \lambda \J(\x_{te};0)^{\top}\operatorname{vec}(\W(0)).
    \end{aligned}
\end{equation}
Combining the two ODE systems, we obtain:
\begin{equation}
\begin{aligned}
    &\frac{\df}{\df t}\left(f_{nn}(\x_{te};\W(t)) - f_{lin}(\x_{te};\W(t))\right) +
    \lambda \left(f_{nn}(\x_{te};\W(t)) - f_{lin}(\x_{te};\W(t))\right) \\
    =&\K_{te}(0)^{\top}\left(f_{lin}(\X;\W(t))-\Y\right) - \K_{te}(t)^{\top}\left(f_{nn}(\X;\W(t))-\Y\right)+\lambda (\J(\x_{te};t) - \J(\x_{te};0))^{\top}\operatorname{vec}(\W(0)).
\end{aligned}
\end{equation}
Then we can bound $|f_{nn}(\x_{te};\W(t)) - f_{lin}(\x_{te};\W(t))|$ for all $t>0$ with probability at least $1-\delta$ as follows:
\begin{equation}
    \begin{aligned}
     & |f_{nn}(\x_{te};\W(t)) - f_{lin}(\x_{te};\W(t))|\\
\leq & \max_{0<s<t}\frac{1}{\lambda}\left|\K_{te}(0)^{\top}\left(f_{lin}(\X;\W(s))-\Y\right) - \K_{te}(s)^{\top}(f_{nn}(\X;\W(s))-\Y)\right| \\
 &+\max_{0<s<t}  \left|(\J(\x_{te};s) - \J(\x_{te};0))^{\top}\operatorname{vec}(\W(0))\right|\\
\leq & \max_{0<s<t} \frac{1}{\lambda}\left(\left|\K_{te}(s)^{\top}\left(f_{nn}(\X;\W(s))-f_{lin}(\X;\W(s))\right)\right| + \left|(\K_{te}(s)-\K_{te}(0))^{\top}(f_{lin}(\X;\W(s))-\Y)\right|\right)\\
&+\max_{0<s<t}  \left|(\J(\x_{te};s) - \J(\x_{te};0))^{\top}\operatorname{vec}(\W(0))\right|\\
\leq & \max_{0<s<t} \frac{1}{\lambda}\left(\|\K_{te}(s)\|_2\|f_{nn}(\X;\W(s))-f_{lin}(\X;\W(s))\|_2 + \|\K_{te}(s)-\K_{te}(0)\|_2\|f_{lin}(\X;\W(s))-\Y\|_2\right)\\
&+\max_{0<s<t}  \left|(\J(\x_{te};s) - \J(\x_{te};0))^{\top}\operatorname{vec}(\W(0))\right|\\
\leq & \frac{\sqrt{n}}{\lambda} \cdot \mathcal{O}\left(\frac{n^{\frac{7}{2}}}{\lambda^2 \delta^2\kappa \sqrt{m}}\right) + \frac{4nR}{\sqrt{2 \pi}\lambda\kappa\delta} \cdot {\sqrt{\frac{n}{\delta}}} + \frac{2\sqrt{m}R^2}{\sqrt{2\pi}\kappa \delta}\\
= & \mathcal{O}\left(\frac{n^4}{\lambda^3 \delta^2 \kappa \sqrt{m}}\right).
    \end{aligned}
\end{equation}
\end{proof}

Now we have built the equivalence between the neural network $f_{nn}(\x)$ and its linearization $f_{lin}(\x)$ in the over-parameterized regime. In the following, we bound the deviation between the NTK predictor $f_{ntk}(\x)$ and $f_{lin}(\x)$. Then according to the triangle inequality, we prove that $\|f_{nn}(\x)-f_{ntk}(\x)\|_2$ can be arbitrarily small when the width $m$ is sufficiently large, hence we build the equivalence between the neural network $f_{nn}(\x)$ and its NTK predictor $f_{ntk}(\x)$.

\noindent\begin{thmbis}{theorem1}
Suppose $0<\lambda< n^{\frac{1}{2}},\ {\kappa = \mathcal{O}\left(\frac{\sqrt{\delta} \lambda \epsilon}{n}\right)}$ and $\ m = \Omega\left(\frac{n^8}{\kappa^2\epsilon^2\delta^4\lambda^6}\right)$, then for any $\x_{te} \in \mathbb{R}^{d}$ with $\|\x_{t e}\|_2=1$, with probability at least $1 - \delta$ over random initialization, we have:
\begin{equation}
    |f_{nn}(\x_{te}) - f_{ntk}(\x_{te})|= \mathcal{O}(\epsilon).
\end{equation}
\end{thmbis}

\begin{proof}
According to the dynamic system of $f_{lin}(\x_{te};\W(t))$, we have:
\begin{equation}
    f_{lin}(\x_{te})=\lim_{t\to\infty}f_{lin}(\x_{te};\W(t)) = f_{lin}(\x_{te};\W(0)) + \K_{te}(0)^{\top}(\K_{tr}(0)+\lambda \mathbf{I}_n)^{-1}(\Y-f_{lin}(\X;\W(0))).
\end{equation}
And the prediction of kernel ridge regression using NTK on a test point $\x_{te}$ is:
\begin{equation}
    f_{ntk}(\x_{te}) = (\K_{te}^{\infty})^{\top} (\K_{tr}^{\infty}+\lambda \mathbf{I}_n)^{-1}\Y.
\end{equation}
Then we can bound $|f_{lin}(\x_{te}) - f_{ntk}(\x_{te})|$ as follows:
\begin{equation}
    \begin{aligned}
&|f_{lin}(\x_{te}) - f_{ntk}(\x_{te})|\\
\leq & |f_{lin}(\x_{te};\W(0))|+| \K_{te}(0)^{\top}(\K_{tr}(0)+\lambda \mathbf{I}_n)^{-1}f_{lin}(\X;\W(0))|\\
&+| \K_{te}(0)^{\top}(\K_{tr}(0)+\lambda \mathbf{I}_n)^{-1}\Y-(\K_{te}^{\infty})^{\top} (\K_{tr}^{\infty}+\lambda \mathbf{I}_n)^{-1}\Y|\\
\leq & |f_{lin}(\x_{te};\W(0))|+\| \K_{te}(0)\|_2\|(\K_{tr}(0)+\lambda \mathbf{I}_n)^{-1}\|_2\|f_{lin}(\X;\W(0))\|_2\\
&+\| \K_{te}(0)^{\top}(\K_{tr}(0)+\lambda \mathbf{I}_n)^{-1}-(\K_{te}^{\infty})^{\top} (\K_{tr}^{\infty}+\lambda \mathbf{I}_n)^{-1}\|_2\|\Y\|_2.
    \end{aligned}
\end{equation}
Note that $f_{lin}(\x;\W(0))$ has zero mean and variance $\mathcal{O}(\kappa^2)$, which implies that $\mathbb{E}\left[f_{lin}(\x;\W(0))^2\right] = \mathcal{O}(\kappa^2)$, hence with probability at least $1-\delta$, we have:
\begin{equation}
    |f_{lin}(\x_{te};\W(0))| = \mathcal{O}\left(\frac{\kappa}{\sqrt{\delta}}\right),\quad \|f_{lin}(\X;\W(0))\|_2  = \mathcal{O}\left(\frac{\sqrt{n}\kappa}{\sqrt{\delta}}\right).
\end{equation}
According to Lemma C.3 in \citep{pmlr-v97-arora19a}, with probability at least $1-\delta$, we have:
\begin{equation}
    \|\K_{tr}(0)-\K_{tr}^{\infty}\|_2 = \mathcal{O}\left(\frac{n \sqrt{\log \frac{n}{\delta}}}{\sqrt{m}}\right),\quad \|\K_{te}(0)-\K_{te}^{\infty}\|_2 = \mathcal{O}\left(\frac{ \sqrt{n\log \frac{n}{\delta}}}{\sqrt{m}}\right).
\end{equation}
Then according to the perturbation theory for linear systems, we have:
\begin{equation}
\begin{aligned}
    &\| \K_{te}(0)^{\top}(\K_{tr}(0)+\lambda \mathbf{I}_n)^{-1}-(\K_{te}^{\infty})^{\top} (\K_{tr}^{\infty}+\lambda \mathbf{I}_n)^{-1}\|_2\\
\leq & \|\K_{tr}(0)+\lambda \mathbf{I}_n)^{-1}\|_2\left(\|\K_{tr}(0)-\K_{tr}^{\infty}\|_2\|(\K_{te}^{\infty})^{\top} (\K_{tr}^{\infty}+\lambda \mathbf{I}_n)^{-1}\|_2+\|\K_{te}(0)-\K_{te}^{\infty}\|_2\right)\\
\leq & \frac{1}{\lambda}\left(\mathcal{O}\left(\frac{n \sqrt{\log \frac{n}{\delta}}}{\sqrt{m}}\right)\mathcal{O}\left(\frac{\sqrt{n}}{\lambda}\right)+\mathcal{O}\left(\frac{ \sqrt{n\log \frac{n}{\delta}}}{\sqrt{m}}\right)\right)\\
=&\mathcal{O}\left(\frac{n^{\frac{3}{2}} \sqrt{\log \frac{n}{\delta}}}{\sqrt{m}\lambda^2}\right).
\end{aligned}
\end{equation}
Hence we can bound $|f_{nn}(\x_{te}) - f_{ntk}(\x_{te})|$ as follows:
\begin{equation}
    \begin{aligned}
        &|f_{nn}(\x_{te}) - f_{ntk}(\x_{te})|\\
     \leq & |f_{nn}(\x_{te}) - f_{lin}(\x_{te})| + |f_{lin}(\x_{te}) - f_{ntk}(\x_{te})|\\
     \leq & \mathcal{O}\left(\frac{n^4}{\lambda^3 \delta^2 \kappa \sqrt{m}}\right)+\mathcal{O}\left(\frac{\kappa}{\sqrt{\delta}}\right)+\mathcal{O}\left(\frac{n\kappa}{\sqrt{\delta}\lambda}\right)+
     \mathcal{O}\left(\frac{n^{2} \sqrt{\log \frac{n}{\delta}}}{\sqrt{m}\lambda^2}\right)\\
     =&\mathcal{O}(\epsilon).
    \end{aligned}
\end{equation}
\end{proof}

Equipped with Theorem \ref{theorem1}, the approximation error can be evaluated through simple analysis as follows.

\noindent\begin{thmbis}{theorem2}
Suppose $0<\lambda< n^{\frac{1}{2}},\ {\kappa = \mathcal{O}\left(\frac{\sqrt{\delta} \lambda \epsilon}{n}\right)}$, $\ m = \Omega\left(\frac{n^8}{\kappa^2\epsilon^2\delta^4\lambda^6}\right)$,  $f_{nn}$ and $f_{ntk}$ are uniformly bounded over the unit sphere by a constant $C$, i.e. $|f_{nn}(\x)|<C$ and $|f_{ntk}(\x)|<C$ for all $\|\x\|_2 =1$. Then for any training point $\x_i \in \X$ and test point $\x_{te} \in \mathbb{R}^{d}$ with $\|\x_{t e}\|_2=1$, with probability at least $1 - \delta$ over random initialization, we have:
\begin{equation}
    |\mathcal{I}_{nn}(\x_i,\x_{te}) - \mathcal{I}_{ntk}(\x_i,\x_{te})|= \mathcal{O}(\epsilon).
\end{equation}
\end{thmbis}

\begin{proof}According to Theorem \ref{theorem1}, we can bound  $|\mathcal{I}_{nn}(\x_i,\x_{te}) - \mathcal{I}_{ntk}(\x_i,\x_{te})|$ with probability at least $1-\delta$ as follows:
\begin{equation}
    \begin{aligned}
       &|\mathcal{I}_{nn}(\x_i,\x_{te}) - \mathcal{I}_{ntk}(\x_i,\x_{te})|\\
     =&\left|\frac{1}{2}[f_{nn}(\x_{te})^2 - f_{nn}^{\backslash i}(\x_{te})^2]+y_{te}[  f_{nn}^{\backslash i}(\x_{te})-f_{nn}(\x_{te})]-\frac{1}{2}[f_{ntk}(\x_{te})^2 - f_{ntk}^{\backslash i}(\x_{te})^2]-y_{te}[  f_{ntk}^{\backslash i}(\x_{te})-f_{ntk}(\x_{te})]\right|\\
     \leq & \frac{1}{2}|f_{nn}(\x_{te})^2 - f_{ntk}(\x_{te})^2| + \frac{1}{2}|f^{\backslash i}_{nn}(\x_{te})^2 - f_{ntk}^{\backslash i}(\x_{te})^2|+|f_{nn}(\x_{te}) - f_{ntk}(\x_{te})|+
     |f_{nn}^{\backslash i}(\x_{te}) - f_{ntk}^{\backslash i}(\x_{te})|\\
     \leq & {C}\mathcal{O}(\epsilon) +  C\mathcal{O}(\epsilon)+\mathcal{O}(\epsilon)+\mathcal{O}(\epsilon)\\
     =& \mathcal{O}(\epsilon).
    \end{aligned}
\end{equation}
\end{proof}

\section{Appendix C: Proofs for Section 6}

\begin{propbis}{pop1}
For any $\x_{te} \in \R^d$ with $\|\x_{te}\|_2=1$, let the width $m \to \infty$, then with
probability arbitrarily close to 1 over random initialization, we have:
\begin{equation}
\begin{aligned}
    \hat{\mathcal{I}}_{nn}(\x_i,\x_{te}) \to& 
    \alpha(\x_i,\x_{te})(f_{ntk}(\x_{te})-y_{te})(f_{ntk}(\x_{i})-y_{i}) \\ \triangleq &\hat{\mathcal{I}}_{ntk}(\x_i,\x_{te}),
\end{aligned}
\end{equation}
where $\alpha(\x_i,\x_{te}) \triangleq (\mathbf{K}_{te}^{\infty})^{\top}(\mathbf{K}_{tr}^{\infty}+\lambda \mathbf{I}_n)^{-1}\mathbf{e}_i$.
\end{propbis}

\begin{proof}
Firstly, we calculate the Hessian matrix $\nabla_{\W}^{2} l\left(\x_{i}; \W(\infty)\right)$ for each $\x_i$:
\begin{equation}
    \begin{aligned}
       &\nabla_{\W}^{2} l\left(\x_{i}; \W(\infty)\right)\\
       =& \frac{\partial}{\partial \W}\left((f_{nn}(\x_i)-y_i) \operatorname{vec}\left(\frac{\partial f_{nn}(\x_i)}{\partial \W}\right)\right)\\
       =&\operatorname{vec}\left(\frac{\partial f_{nn}(\x_i)}{\partial \W}\right)\operatorname{vec}\left(\frac{\partial f_{nn}(\x_i)}{\partial \W}\right)^{\top} + (f_{nn}(\x_i)-y_i)\nabla_{\W}^{2}f_{nn}(\x_i)\\
       =&\operatorname{vec}\left(\frac{\partial f_{nn}(\x_i)}{\partial \W}\right)\operatorname{vec}\left(\frac{\partial f_{nn}(\x_i)}{\partial \W}\right)^{\top}.
    \end{aligned}
\end{equation}
Notice that the term $\nabla_{\W}^{2}f_{nn}(\x_i)$ vanishes because of the piecewise linear property of ReLU function. Thus we can calculate $ \hat{\mathcal{I}}_{nn}(\x_i,\x_{te})$ as follows:
\begin{equation}
    \begin{aligned}
     &\quad\hat{\mathcal{I}}_{nn}(\x_i,\x_{te}) \\&=\lim_{t\to \infty}\operatorname{vec}\left(\frac{\partial l(\x_{te};\W(t))}{\partial \W(t)}\right)^{\top}H(t)^{-1}\operatorname{vec}\left(\frac{\partial l(\x_{i};\W(t))}{\partial \W(t)}\right) \\
       & =\lim_{t\to \infty}(f_{nn}(\x_{te})-y_{te})(f_{nn}(\x_{i})-y_{i})\operatorname{vec}\left(\frac{\partial f_{nn}(\x_{te})}{\partial \W}\right)^{\top}\left(\J(\X;t)\J(\X;t)^{\top}+\lambda \mathbf{I}_{md}\right)^{-1}\operatorname{vec}\left(\frac{\partial f_{nn}(\x_{i})}{\partial \W}\right)\\
       & =\lim_{t\to \infty}(f_{nn}(\x_{te})-y_{te})(f_{nn}(\x_{i})-y_{i})\J({\x_{te}};t)^{\top}\left(\J(\X;t)\J(\X;t)^{\top}+\lambda \mathbf{I}_{md}\right)^{-1}\J(\X;t)\mathbf{e}_i\\
       & =\lim_{t\to \infty}(f_{nn}(\x_{te})-y_{te})(f_{nn}(\x_{i})-y_{i})\J(\x_{te};t)^{\top}\J(\X;t)\left(\J(\X;t)^{\top}\J(\X;t)+\lambda \mathbf{I}_{n}\right)^{-1}\mathbf{e}_i\\
        &=\lim_{t\to \infty}(f_{nn}(\x_{te})-y_{te})(f_{nn}(\x_{i})-y_{i})\mathbf{K}_{te}(t)^{\top}\left(\mathbf{K}_{tr}(t)+\lambda \mathbf{I}_n\right)^{-1}\mathbf{e}_i\\
        &\overset{m\to \infty}{\longrightarrow} (f_{ntk}(\x_{te})-y_{te})(f_{ntk}(\x_{i})-y_{i})(\mathbf{K}_{te}^{\infty})^{\top}(\mathbf{K}_{tr}^{\infty}+\lambda \mathbf{I}_n)^{-1}\mathbf{e}_i\\
        &\triangleq\hat{\mathcal{I}}_{ntk}(\x_i,\x_{te}).
    \end{aligned}
\end{equation}
\end{proof}

\begin{propbis}{pop2}
For any $\x_{te} \in \R^d$ with $\|\x_{te}\|_2=1$ and $\x_i \in \X$, we have:
\begin{equation}
\begin{aligned}
    \mathcal{I}_{ntk}(\x_i,\x_{te}) &=  \underbrace{\alpha(\x_i,\x_{te})(f_{ntk}(\x_{te})-y_{te})(f_{ntk}^{\backslash i}(\x_{i})-y_{i})}_{\text{\uppercase\expandafter{\romannumeral1}}} \\ 
    &+\underbrace{\frac{1}{2}\alpha(\x_i,\x_{te})^2(f_{ntk}^{\backslash i}(\x_i) - y_i)^2}_{\text{\uppercase\expandafter{\romannumeral2}}}.
\end{aligned}
\end{equation}
\end{propbis}

\begin{proof}
Firstly, we calculate $\mathcal{I}_{ntk}(\x_i, \x_{te})$ for each $\x_i$ as follows:
\begin{equation}
    \begin{aligned}
        &\mathcal{I}_{ntk}(\x_i,\x_{te}) \\=& \frac{1}{2}(f_{ntk}^{\backslash i}(\x_{te})-y_{te})^2-\frac{1}{2}(f_{ntk}(\x_{te})-y_{te})^2\\
        =&\frac{1}{2}((f_{ntk}(\x_{te})+\Delta f)-y_{te})^2-\frac{1}{2}(f_{ntk}(\x_{te})-y_{te})^2\\
        =&\Delta f(f_{ntk}(\x_{te})-y_{te}) + \frac{1}{2}\Delta f^2,
    \end{aligned}
\end{equation}
where $\Delta f \triangleq f_{ntk}^{\backslash i}(\x_{te}) - f_{ntk}(\x_{te}) = -\alpha(\x_i,\x_{te})\frac{\mathbf{k}_{-i}^{\top}\mathcal{Y}}{k_{-ii}}$. Now we calculate $f_{ntk}^{\backslash i}(\x_i) - y_i$ for each $\x_i$ as follows:
\begin{equation}
    \begin{aligned}
        &f_{ntk}^{\backslash i}(\x_i) - y_i\\
       =&\mathbf{e}_i^{\top}\mathbf{K}_{tr}^{\infty}\left((\mathbf{K}_{tr}^{\infty}+\lambda \mathbf{I}_n)^{-1} - \frac{\mathbf{k}_{-i}\mathbf{k}_{-i}^{\top}}{k_{-ii}}\right)\mathcal{Y} - y_i\\
       =&\mathbf{e}_i(\mathbf{K}_{tr}^{\infty}+\lambda \mathbf{I}_n)(\mathbf{K}_{tr}^{\infty}+\lambda \mathbf{I}_n)^{-1}\mathcal{Y} - \lambda \mathbf{e}_i^{\top}(\mathbf{K}_{tr}^{\infty}+\lambda \mathbf{I}_n)^{-1}\mathcal{Y} - \mathbf{e}_{i}^{\top}\K_{tr}^{\infty}\frac{\mathbf{k}_{-i}\mathbf{k}_{-i}^{\top}}{k_{-ii}}\mathcal{Y} - y_i\\
       =&y_i - \lambda \mathbf{k}_{-i}^{\top}\mathcal{Y} - \mathbf{e}_{i}^{\top}(\K_{tr}^{\infty}+\lambda \mathbf{I}_n)\frac{\mathbf{k}_{-i}\mathbf{k}_{-i}^{\top}}{k_{-ii}}\mathcal{Y} + \lambda \mathbf{e}_i^{\top}\frac{\mathbf{k}_{-i}\mathbf{k}_{-i}^{\top}}{k_{-ii}}\mathcal{Y}- y_i\\
       =&-\lambda \mathbf{k}_{-i}^{\top}\mathcal{Y} - \frac{\mathbf{k}_{-i}^{\top}\mathcal{Y}}{k_{-ii}}+\lambda \mathbf{k}_{-i}^{\top}\mathcal{Y}\\
       =&-\frac{\mathbf{k}_{-i}^{\top}\mathcal{Y}}{k_{-ii}}.
    \end{aligned}
    \end{equation}
Thus we have $\Delta f = \alpha(\x_i, \x_{te})(f_{ntk}^{\backslash i}(\x_i) - y_i)$ and finish the proof.
\end{proof}

\noindent\begin{thmbis}{lowerbound}
Given $\x_i \in \X$ and $\x_{te} \in \R^d$ with $\|\x_{te}\|_2=1$, we have:
\begin{equation}
\begin{aligned}
     &\left|\mathcal{I}_{ntk}(\x_i,\x_{te})- \hat{\mathcal{I}}_{ntk}(\x_i,\x_{te})\right| \\\geq& \frac{\lambda_{min}}{\lambda_{min}+\lambda}|\mathcal{I}_{ntk}(\x_i,\x_{te})|-\frac{1}{2}\alpha(\x_i,\x_{te})^2(f_{ntk}^{\backslash i}(\x_i) - y_i)^2,
\end{aligned}
\end{equation}
where $\lambda_{min}$ is the least eigenvalue of $\K_{tr}^{\infty}$. Furthermore, if $|{\alpha}(\x_i,\x_{te})| \leq \sqrt{\frac{\gamma}{n}}\|\boldsymbol{\alpha}(\x_{te}) \|_2$ for some $\gamma>0$, where $\boldsymbol{\alpha}(\x_{te}) \triangleq (\mathbf{K}_{te}^{\infty})^{\top}(\mathbf{K}_{tr}^{\infty}+\lambda \mathbf{I}_n)^{-1}$, and $\mathbb{E}_{\x\sim\mathcal{D}}[(f_{ntk}(\x)-y)^2]=\mathcal{O}({\sqrt{1/n}})$, then with probability at least $1-\delta$ we have:
\begin{equation}
\begin{aligned}
  &\left|\mathcal{I}_{ntk}(\x_i,\x_{te})- \hat{\mathcal{I}}_{ntk}(\x_i,\x_{te})\right|\\\geq& \frac{\lambda_{min}}{\lambda_{min}+\lambda}|\mathcal{I}_{ntk}(\x_i,\x_{te})|-\mathcal{O}(\frac{\gamma}{\delta\lambda n^{3/2}}).
\end{aligned}
\end{equation}
\end{thmbis}

\begin{proof}
Firstly, we give the lower bound of $A_i\triangleq \mathbf{e}_i^{\top}\K_{tr}^{\infty}(\K_{tr}^{\infty}+\lambda \mathbf{I}_n)^{-1}\mathbf{e}_i$ as follows:
\begin{equation}
    \begin{aligned}
       A_i=& \mathbf{e}_i^{\top}\K_{tr}^{\infty}(\K_{tr}^{\infty}+\lambda \mathbf{I}_n)^{-1}\mathbf{e}_i\\
       =&\mathbf{e}_i^{\top}\mathbf{U}_{tr}\mathbf{\Lambda}_{tr}(\mathbf{\Lambda}_{tr}+\lambda \mathbf{I}_n)^{-1}\mathbf{U}_{tr}^{\top}\mathbf{e}_i\\
       \geq& \frac{\lambda_{min}}{\lambda_{min}+\lambda},
    \end{aligned}
\end{equation}
where $\mathbf{U}_{tr}\mathbf{\Lambda}_{tr}\mathbf{U}_{tr}^{\top} = \K_{tr}^{\infty}$ and $\lambda_{min}$ denote the spectral decomposition and the least eigenvalue of $\K_{tr}^{\infty}$ respectively. 
Let us recall the leave-one-out theory on kernel ridge regression as Lemma 4.1 described in \cite{elisseeff2003leave}, for all $\x_i\in \X$, we have 
\begin{equation}
    f_{ntk}(\x_i) - y_i = (f_{ntk}^{\backslash i}(\x_i)-y_i)(1-A_i).
\end{equation}
Thus we can bound $\left|\mathcal{I}_{ntk}(\x_i,\x_{te})- \hat{\mathcal{I}}_{ntk}(\x_i,\x_{te})\right|$ as follows:
\begin{equation}
    \begin{aligned}
        &\left|\mathcal{I}_{ntk}(\x_i,\x_{te})- \hat{\mathcal{I}}_{ntk}(\x_i,\x_{te})\right|\\
        =&\left|A_i\alpha(\x_i,\x_{te})(f_{ntk}(\x_{te})-y_{te})(f_{ntk}^{\backslash i}(\x_{i})-y_i) + \frac{1}{2}\alpha(\x_i,\x_{te})^2(f_{ntk}^{\backslash i}(\x_i) - y_i)^2\right|\\
        =&\left|A_i\mathcal{I}_{ntk}(\x_i,\x_{te})+\frac{1-A_i}{2}\alpha(\x_i,\x_{te})^2(f_{ntk}^{\backslash i}(\x_i) - y_i)^2\right|\\
        \geq& \left|A_i\mathcal{I}_{ntk}(\x_i,\x_{te})\right|-\left|\frac{1-A_i}{2}\alpha(\x_i,\x_{te})^2(f_{ntk}^{\backslash i}(\x_i) - y_i)^2\right|\\
        \geq& \frac{\lambda_{min}}{\lambda_{min}+\lambda}|\mathcal{I}_{ntk}(\x_i,\x_{te})|-\frac{1}{2}\alpha(\x_i,\x_{te})^2(f_{ntk}^{\backslash i}(\x_i) - y_i)^2.
        \end{aligned}
\end{equation}
Now we consider the magnitude of $\mathbf{\alpha}(\x_i,\x_{te})^2$. For a two-layer neural network $f_{nn}(\x)$ with $m$ neurons, we define $\alpha_m(\x_i,\x_{te})$ as follows:
\begin{equation}
    \alpha_m(\x_i,\x_{te}) \triangleq \lim_{t \to \infty}\J(\x_{te};t)^{\top}\J(\X;t)\left(\J(\X;t)^{\top}\J(\X;t)+\lambda \mathbf{I}_{n}\right)^{-1}\mathbf{e}_i.
\end{equation}
On the one hand, we have $\lim_{m \to \infty} \alpha_m(\x_i,\x_{te}) =\alpha(\x_i,\x_{te})$ according to Theorem \ref{theorem1}. On the other hand, $\boldsymbol{\alpha}_m(\x_{te}) \triangleq [\alpha_m(\x_1,\x_{te}),\cdots,\alpha_m(\x_n,\x_{te}) ]^{\top}$ can be regarded as the optimizer of the following optimization problem:
\begin{equation}
    \min_{\boldsymbol{\alpha}\in \R^n}\  \lim_{t\to \infty}\frac{1}{2}\left\|\J(\x_{te};t)  - \J(\X;t)\boldsymbol{\alpha} \right\|_2^2+\frac{\lambda}{2}\left\|\boldsymbol{\alpha}\right\|_2^2.
\end{equation}
Now we consider $\overline{\boldsymbol{\alpha}}\triangleq \left[0,\cdots, 0\right]^{\top}\in \R^n$, then we have:
\begin{equation}
\lim_{m\to \infty}\lim_{t\to \infty}\frac{1}{2}\left\|\J(\x_{te};t)  - \J(\X;t)\overline{\boldsymbol{\alpha}} \right\|_2^2+\frac{\lambda}{2}\left\|\overline{\boldsymbol{\alpha}}\right\|_2^2 =\frac{1}{4}.
\end{equation}
Thus for sufficiently large $m$, we have:
\begin{equation}
    \frac{\lambda}{2}\|\boldsymbol{\alpha}_m(\x_{te})\|_2^2 \leq \lim_{t\to \infty}\frac{1}{2}\left\|\J(\x_{te};t)  - \J(\X;t)\boldsymbol{\alpha}_m(\x_{te}) \right\|_2^2+\frac{\lambda}{2}\left\|\boldsymbol{\alpha}_m(\x_{te})\right\|_2^2 \leq \frac{1}{2}.
\end{equation}
Thus we have:
\begin{equation}
  \mathbf{\alpha}(\x_i,\x_{te})^2\leq {\frac{\gamma}{n}}\|\boldsymbol{\alpha}(\x_{te})\|_2^2\leq \frac{ \gamma}{\lambda n}.
\end{equation}
Furthermore, when the generalization error of $f_{ntk}$ on $\mathcal{D}$ satisfies $\mathbb{E}_{\x\sim\mathcal{D}}[(f_{ntk}(\x)-y)^2]=\mathcal{O}({\sqrt{1/n}})$, with probability at least $1-\delta$, we have:
\begin{equation}
    \left(f_{ntk}^{\backslash i}(\x_i)-y_i\right)^2 \leq  \frac{\mathbb{E}_{\x_i\in \X}\left[\left(f_{ntk}^{\backslash i}(\x_i)-y_i\right)^2\right]}{\delta} = \mathcal{O}\left(\sqrt{\frac{1}{n\delta^2}}\right).
\end{equation}
Thus if $|{\alpha}(\x_i,\x_{te})| \leq \sqrt{\frac{\gamma}{n}}\|\boldsymbol{\alpha}(\x_{te}) \|_2$ and $\mathbb{E}_{\x\sim\mathcal{D}}[(f_{ntk}(\x)-y)^2]=\mathcal{O}({\sqrt{1/n}})$, then with probability at least $1-\delta$ we have:
\begin{equation}
       \left|\mathcal{I}_{ntk}(\x_i,\x_{te})- \hat{\mathcal{I}}_{ntk}(\x_i,\x_{te})\right|\geq \frac{\lambda_{min}}{\lambda_{min}+\lambda}|\mathcal{I}_{ntk}(\x_i,\x_{te})|-\mathcal{O}(\frac{\gamma}{\delta\lambda n^{3/2}}).
\end{equation}
\end{proof}

\noindent\begin{thmbis}{theorem4}
Consider a dataset $(\X,\Y) = \{\x_i,y_i\}_{i=1}^n$ generated from the finite mixture model described in Definition 1. Let $\boldsymbol{\alpha}(\x_i) \triangleq (\mathbf{K}_{tr}^{\infty})^{\top}(\mathbf{K}_{tr}^{\infty}+\lambda \mathbf{I}_n)^{-1}\mathbf{e}_i$. Suppose $|[\boldsymbol{\alpha}(\x_i)]_i|\leq \sqrt{\frac{\gamma}{n}}\|\boldsymbol{\alpha}(\x_{i}) \|_2$ for some $\gamma>0$, and $\lambda > \frac{\sqrt{2}\lambda_{max}\epsilon_r}{1-\sqrt{2}\epsilon_r}$, where $\epsilon_r^2 \triangleq 2r^2+\arccos(1-2r^2)$ is a small constant. Then for $\x_i \in \X_k$ and $\x_{te}\sim \mathcal{D}$, we have:
\begin{equation}
\begin{aligned}
         & \left|\mathcal{I}_{ntk}(\x_i,\x_{te})- \hat{\mathcal{I}}_{ntk}(\x_i,\x_{te})\right| \\\leq&
    \underbrace{\sqrt{\frac{\gamma}{n^2p_k}}|\mathcal{I}_{ntk}(\x_i,\x_{te})|}_{\text{{\uppercase\expandafter{\romannumeral1}}}} +   \underbrace{\frac{1}{2}\alpha(\x_i,\x_{te})^2(f_{ntk}^{\backslash i}(\x_i) - y_i)^2}_{\text{\uppercase\expandafter{\romannumeral2}}}.
\end{aligned}
\end{equation}
\end{thmbis}

\begin{proof}
Similar to the proof in Theorem \ref{lowerbound}, we bound $\left|\mathcal{I}_{ntk}(\x_i,\x_{te})- \hat{\mathcal{I}}_{ntk}(\x_i,\x_{te})\right|$ as follows: 
\begin{equation}
    \begin{aligned}
        &\left|\mathcal{I}_{ntk}(\x_i,\x_{te})- \hat{\mathcal{I}}_{ntk}(\x_i,\x_{te})\right|\\
        =&\left|A_i\mathcal{I}_{ntk}(\x_i,\x_{te})+\frac{1-A_i}{2}\alpha(\x_i,\x_{te})^2(f_{ntk}^{\backslash i}(\x_i) - y_i)^2\right|\\
        \leq& \left|A_i\mathcal{I}_{ntk}(\x_i,\x_{te})\right|+\left|\frac{1-A_i}{2}\alpha(\x_i,\x_{te})^2(f_{ntk}^{\backslash i}(\x_i) - y_i)^2\right|\\
        \leq& |A_i|\cdot|\mathcal{I}_{ntk}(\x_i,\x_{te})|+\frac{1}{2}\alpha(\x_i,\x_{te})^2(f_{ntk}^{\backslash i}(\x_i) - y_i)^2.
        \end{aligned}
\end{equation}
For $\x_i$ sampled from $\mathcal{D}_k$, note that the approximation error is controlled by $|A_i|$ and in the following we show that $|A_i|$ decreases with probability density of $\mathcal{D}_k$, i,e., $p_k$.
Firstly, we consider $\boldsymbol{\alpha}_m(\x_i)\triangleq [\mathbf{\alpha}_m(\x_1,\x_i),\cdots,\mathbf{\alpha}_m(\x_n,\x_i)]^{\top}$, which is the optimizer of the following unconstrained optimization problem:
\begin{equation}
    \min_{\boldsymbol{\alpha}\in \R^n}\  \lim_{t\to \infty}\frac{1}{2}\left\|\J(\x_i;t)  - \J(\X;t)\mathbf{\boldsymbol{\alpha}} \right\|_2^2+\frac{\lambda}{2}\left\|\boldsymbol{\alpha}\right\|_2^2. \label{opt1}
\end{equation}
And there exists some $\epsilon_{\lambda}>0$, such that the following constrained optimization problem is equivalent to the problem (\ref{opt1}):
\begin{equation}
\begin{aligned}
   &\quad\quad \min_{\boldsymbol{\alpha}\in \R^n} \|\boldsymbol{\alpha}\|_2^2\\
    s.t.&\ \left\|\J(\x_i;t)  - \J(\X;t)\boldsymbol{\alpha} \right\|_2^2\leq \epsilon_{\lambda}^2.
\end{aligned}    \label{opt2}
\end{equation}
On the one hand, $\boldsymbol{\alpha}_m(\x_i)$ is the optimizer of the problem (\ref{opt2}), thus we have:
\begin{equation}
\begin{aligned}
        \epsilon_{\lambda}^2& \geq \left\|\J(\x_i;t)  - \J(\X;t)\boldsymbol{\alpha}_m(\x_i) \right\|_2^2 \\
        &=\left\|\J(\x_i;t)  - \J(\X;t)\left(\J(\X;t)^{\top}\J(\X;t)+\lambda \mathbf{I}_{n}\right)^{-1}\J(\X;t)^{\top}\J(\x_i;t) \right\|_2^2\\
        &\geq \left\|\mathbf{I}_{md}-\J(\X;t)\left(\J(\X;t)^{\top}\J(\X;t)+\lambda \mathbf{I}_{n}\right)^{-1}\J(\X;t)^{\top}\right\|_2^2\|\J(\x_i;t) \|_2^2\\
        &\geq \frac{1}{2}\left(\frac{\lambda}{\lambda+\lambda_{max}}\right)^2\quad (m\to \infty).
\end{aligned}
\end{equation}
On the other hand, we define the vector $\hat{\boldsymbol{\alpha}}$ as follows:
\begin{equation}
    [\hat{\boldsymbol{\alpha}}]_i=\left\{
\begin{aligned}
& \frac{1}{n_k} , & \mbox{if}\  \x_i \in \mathcal{X}_k ; \\
&0  , & \mbox{otherwise}.
\end{aligned}
\right.
\end{equation}
Notice that for $\x$, $\x^{\prime}$ generated from the same sub-distribution $\mathcal{D}_k$, according to the assumption \textbf{(A2)} of dataset, we have $\|\x-\x^{\prime}\|_2^2 \leq 4r^2$, and we can bound $\lim_{m\to \infty}\lim_{t\to\infty}\|\mathbf{J}(\x;t)-\mathbf{J}(\x^{\prime};t)\|_2^2$ as follows:
\begin{equation}
    \lim_{m\to \infty}\lim_{t\to\infty}\|\mathbf{J}(\x;t)-\mathbf{J}(\x^{\prime};t)\|_2^2 = 1-\frac{\mathbf{x}^{\top} \mathbf{x}^{\prime}\left(\pi-\arccos \left(\mathbf{x}^{\top} \mathbf{x}^{\prime}\right)\right)}{\pi}\leq 2r^2+\arccos(1-2r^2)\triangleq \epsilon^{2}_{r}.
\end{equation}
Thus for $\hat{\boldsymbol{\alpha}}$, we have $\lim_{m\to \infty}\lim_{t\to\infty}\|\mathbf{J}(\x_i;t)-\mathbf{J}(\X;t)\hat{\boldsymbol{\alpha}}\|_2^2 \leq \epsilon_r^2$. Notice that we set $\lambda\geq \frac{\sqrt{2}{\lambda}_{max}\epsilon_r}{1-\sqrt{2}\epsilon_r}$, thus we have $\epsilon_r \leq\epsilon_{\lambda}$, which means that $\hat{\boldsymbol{\alpha}}$ is a feasible solution of the problem (\ref{opt2}). Thus we have:
\begin{equation}
    \begin{aligned}
        |A_i| &= \lim_{m\to\infty}\lim_{t \to \infty}|[\boldsymbol{\alpha}_m(\x_i)]_i| \\
        &\leq \sqrt{\frac{\gamma}{n}}\lim_{m\to\infty}\lim_{t \to \infty}\|\boldsymbol{\alpha}_m(\x_i)\|_{2}\\
        &\leq\sqrt{\frac{\gamma}{n}}\|\hat{\boldsymbol{\alpha}}\|_2 \\
        &=  \sqrt{\frac{\gamma}{n\cdot n_k}}.
    \end{aligned}
\end{equation}
Hence we have:
\begin{equation}
    \left|\mathcal{I}_{ntk}(\x_i,\x_{te})- \hat{\mathcal{I}}_{ntk}(\x_i,\x_{te})\right| \leq 
     \sqrt{\frac{\gamma}{n^2p_k}}|\mathcal{I}_{ntk}(\x_i,\x_{te})|+ \frac{1}{2}\alpha(\x_i,\x_{te})^2(f_{ntk}^{\backslash i}(\x_i) - y_i)^2.
\end{equation}
\end{proof}

\end{document}